%% file: main.tex
\documentclass{article}

\usepackage{microtype}
\usepackage{graphicx}
\usepackage{subfigure}
\usepackage{booktabs} %

\usepackage{hyperref}

\usepackage[accepted]{icml2024}

\usepackage{amsmath}
\usepackage{amssymb}
\usepackage{mathtools}
\usepackage{amsthm}
\usepackage{dsfont}
\usepackage{float}

\usepackage[capitalize,noabbrev]{cleveref}

\theoremstyle{plain}
\newtheorem{theorem}{Theorem}[section]
\newtheorem{proposition}[theorem]{Proposition}

\theoremstyle{definition}

\theoremstyle{remark}

\usepackage[textsize=tiny]{todonotes}

\input{math_commands.tex}

\icmltitlerunning{Robust Inverse Graphics via Probabilistic Inference}

\begin{document}

\twocolumn[
\icmltitle{Robust Inverse Graphics via Probabilistic Inference}

\icmlsetsymbol{equal}{*}

\begin{icmlauthorlist}
\icmlauthor{Tuan Anh Le}{equal,goog}
\icmlauthor{Pavel Sountsov}{equal,goog}
\icmlauthor{Matthew D. Hoffman}{goog}
\icmlauthor{Ben Lee}{goog}
\icmlauthor{Brian Patton}{goog}
\icmlauthor{Rif A. Saurous}{goog}
\end{icmlauthorlist}

\icmlaffiliation{goog}{Google}

\icmlcorrespondingauthor{Tuan Anh Le}{tuananhl@google.com}
\icmlcorrespondingauthor{Pavel Sountsov}{siege@google.com}

\icmlkeywords{inverse graphics; probabilistic inference; diffusion models; NeRF; sequential Monte Carlo}

\vskip 0.3in
]

\printAffiliationsAndNotice{\icmlEqualContribution} %

\begin{abstract}
How do we infer a 3D scene from a single image in the presence of corruptions like rain, snow or fog?
Straightforward domain randomization relies on knowing the family of corruptions ahead of time.
Here, we propose a Bayesian approach---dubbed robust inverse graphics (RIG)---that relies on a strong scene prior and an uninformative uniform corruption prior, making it applicable to a wide range of corruptions.
Given a single image, RIG performs posterior inference jointly over the scene and the corruption.
We demonstrate this idea by training a neural radiance field (NeRF) scene prior and using a secondary NeRF to represent the corruptions over which we place an uninformative prior.
RIG, trained only on clean data, outperforms depth estimators and alternative NeRF approaches that perform point estimation instead of full inference.
The results hold for a number of scene prior architectures based on normalizing flows and diffusion models.
For the latter, we develop \emph{reconstruction-guidance with auxiliary latents} (ReGAL)---a diffusion conditioning algorithm that is applicable in the presence of auxiliary latent variables such as the corruption.
RIG demonstrates how scene priors can be used beyond generation tasks.
\end{abstract}

\section{Introduction}
We explore the inference of a 3D scene from a single image that is robust to corruptions to the underlying 3D scene and its measurement, such as the presence of rain, snow, fog or other floaters, or imperfect knowledge of camera parameters.
Being robust to such corruptions would extend the range of normal operation of such a system.
For example, a self-driving car would be able to drive in more diverse weather conditions.

Current approaches either use domain randomization \citep{zhao2022unsupervised, gasperini2023robust, zhu2023ecdepth, saunders2023selfsupervised} or regularize training using additional loss terms \citep{wynn2023diffusionerf, warburg2023nerfbusters}.
Data randomization involves selecting a family of corruptions which are included in the data generation process.
However, we don't always know the kind of corruptions we want to be robust to.
Floaters---inaccurate high density regions in the reconstructed scenes---are difficult to predict ahead of time, so are commonly handled by including additional loss terms in the reconstruction loss.
It is unclear how to extend this approach to other possibly more extreme situations.

We propose \emph{robust inverse graphics} (RIG), where we view the problem through the lens of probabilistic inference.
We rely on a pre-trained scene prior, in our case a prior over neural radiance fields (NeRFs)\footnote{We require a 3D representation which supports differentiable rendering and is convenient to place a distribution over. NeRFs were the best at fulfilling these requirements at the time of writing, but other approaches could also work, such as Gaussian splatting.}, and a weak prior over corruptions---in the case of 3D scene corruptions, another NeRF with a uniform prior over its parameters.
RIG performs full probabilistic inference over both the scene and corruption NeRF parameters instead of searching for the most probable solution, such as a single set of NeRF parameters obtained via maximum a posteriori (MAP) inference.

\begin{figure}[!tb]
  \centering
    \includegraphics[scale=0.75]{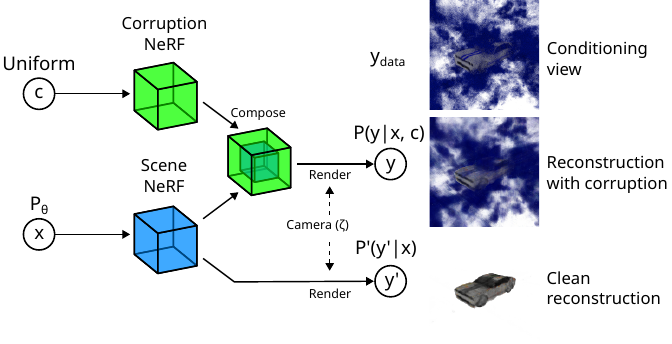}
    \vspace{-2em}
\caption{
    Robust Inverse Graphics (RIG).
    By modeling the generative process of 2D renderings $y$ of 3D scenes, we can reconstruct clean scenes by performing joint probabilistic inference on scene latents ($x$) and corruption parameters ($c$).
}
    \vspace{-1em}
    \label{fig:diagram}
\end{figure}

Fine-tuning by minimizing reconstruction error is a standard technique in conditional 3D generation  \citep[e.g.][]{chen2023singlestage}.
In the presence of priors, fine-tuning corresponds to obtaining a MAP estimate of the NeRF parameters of the scene and the NeRF parameters of the observed corruption.
While MAP may yield good point estimates in a multi-view setting, in a single-image setting it yields a ``billboard'' solution where the corruption NeRF ends up explaining the scene from the conditioned-on view at the expense of the scene NeRF---the corruption NeRF can become a billboard in front of the camera.
Below, we prove that this problem is intrinsic to MAP inference, and empirically show that doing full probabilistic inference doesn't suffer from this problem.

We make the following core contributions:

\begin{enumerate}
    \item
    We propose RIG, a general framework for robust inverse graphics via performing full probabilistic inference over scene and corruption latents.
    
    \item
    We validate the RIG approach on 3D datasets with a number of prior and NeRF representations, across a number of possible corruptions.
    We empirically show that full probabilistic inference produces better results than point estimates on the monocular depth
    
    \item
    To enable RIG on diffusion-based priors, we develop reconstruction-guidance with auxiliary latents (ReGAL) and its importance sampling generalization---a class of general-purpose diffusion conditioning methods that is applicable to latent variable models where a subset of the latents is modeled using a diffusion prior. To the best of our knowledge, we are the first to consider diffusion conditioning in this setting.
\end{enumerate}

\section{Method}

Given a single image $y$, we would like to infer the underlying 3D scene representation $x$ over which we have a prior $p(x)$.
We assume the scene contains a corruption $c$ with a corresponding prior $p(c)$ and that we have a rendering function $R(x, c, \zeta)$ that produces an image given the scene, the corruption and camera parameters $\zeta$.
We form the likelihood $p(y | x, c)$ (we omit $\zeta$ for notational clarity) by treating the rendered image as the mean of per-pixel-and-channel independent Gaussians with a constant-variance observation noise.
Our approach, dubbed robust inverse graphics (RIG), performs full posterior inference $p(x, c | y)$ to obtain the scene. \Cref{fig:diagram} illustrates the approach.

\textbf{Scene representation}
We focus on neural radiance field (NeRF) representations \citep{mildenhall2021nerf} because of their amenability to gradient-based inference.
Let $f(x, x_r, d_r)$ be a trainable function\footnote{For a fixed $x$, $f(x, \cdot, \cdot)$ is a radiance field.} mapping from the scene latent $x$, ray position $x_r \in \mathbb R^3$ and ray direction $d_r \in \mathbb S^2$ to color $\gamma \in [0, 1]^3$ and density $\sigma \in \mathbb R^+$.
Given camera extrinsics and intrinsics $\zeta$, we render the color of a pixel by casting a ray from the camera origin through that pixel's center and evaluating the volumetric rendering integral along that ray.
In practice, this integral is estimated using stochastic quadrature by querying the integrated radiance field along the ray, using hierarchical sampling described in \citet{barron2021mipnerf,barron2023zipnerf}.
Rendering the full image, denoted $y = R(x)$ (and omitting the camera parameters $\zeta$), is done by volumetrically rendering each pixel in the image. See \Cref{app:nerf_rendering} for details.

\textbf{Scene prior}
We assume we have a pre-trained prior $p(x)$ over NeRFs from which we can sample the scene latent $x$ and render images $y$ from different viewpoints $\zeta$.
In our experiments, we use the ProbNeRF model \citep{hoffman2023probnerf} which places a RealNVP \citep{dinh2016density} prior over $x$ and trains a hypernetwork mapping from $x$ to the weights of a multi-layer perceptron that maps from $(x_r, d_r)$ to $(\gamma, \sigma)$.
One advantage of the ProbNeRF model is that it is easy to evaluate the prior density $p(x)$.
We also validate our method on more powerful denoising diffusion priors (\Cref{sec:diffusion_scene_priors}) where $p(x)$ is not easy to evaluate, for which we develop a novel inference algorithm (\Cref{sec:diffusion_conditioning}).

\textbf{Corruption representation and prior}
We focus on corruptions to the 3D scenes such as floaters or weather artifacts like rain, snow or fog, although our approach generalizes to sensor corruptions like camera intrinsics noise (\Cref{sec:results}).
\Cref{fig:example_corruptions} shows some examples of how corruptions affect the observed images.
For camera intrinsics, we focus on inferring the field-of-view (FOV), which is one of the camera intrinsic parameters of the pinhole camera parameterization.
We represent the 3D corruptions $c$ as parameters of another NeRF. 
Unlike the scene $x$, we don't require a strong prior over $c$.
In our experiments, we assume an improper prior $p(c) \propto 1$.
This means that we don't need to know the family of corruptions ahead of time; the corruption can be any 3D entity ranging from weather artifacts and floaters to unwanted objects.

\begin{figure}[!tb]
  \centering
    \includegraphics[scale=0.55]{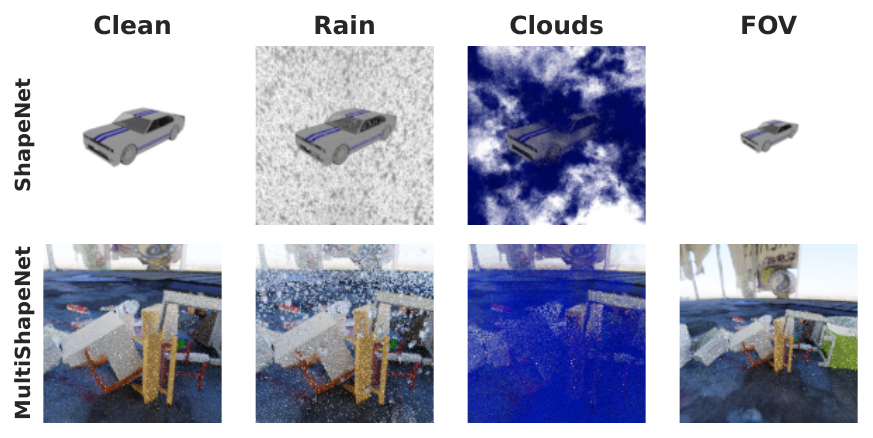}
    \vspace{-1em}
    \caption{
        Example corruptions. FOV refers to the field-of-view intrinsic parameter of the pinhole camera parameterization.
    }
    \label{fig:example_corruptions}
    \vspace{-2em}
\end{figure}

\textbf{Likelihood}
To render an image $y$ given the scene latent $x$ and corruption $c$, we compose the respective NeRF outputs.
Given a ray position and direction $(x_r, d_r)$ we compose the outputs of the scene NeRF $(\gamma_z, \sigma_z)$ and corruption NeRF $(\gamma_c, \sigma_c)$ as $\sigma = \sigma_z + \sigma_c, \gamma = (\gamma_z \sigma_z  + \gamma_c \sigma_c) / \sigma$ \citep{niemeyer2021giraffe}.
We denote rendering of the composed NeRF as $y = R(x, c)$.
The likelihood is a per-pixel-and-channel Gaussian $p(y | x, c) = \prod_{\text{pixel } i \text{ and channel } j} \mathcal N(y_{ij} | R(x, c)_{ij}, \sigma_y^2)$ where $\sigma_y^2$ is the observation noise variance.

\begin{figure}[!tb]
  \centering
    \includegraphics[scale=0.75]{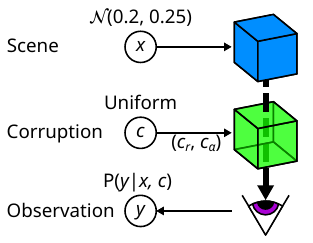}
    \includegraphics[scale=0.55]{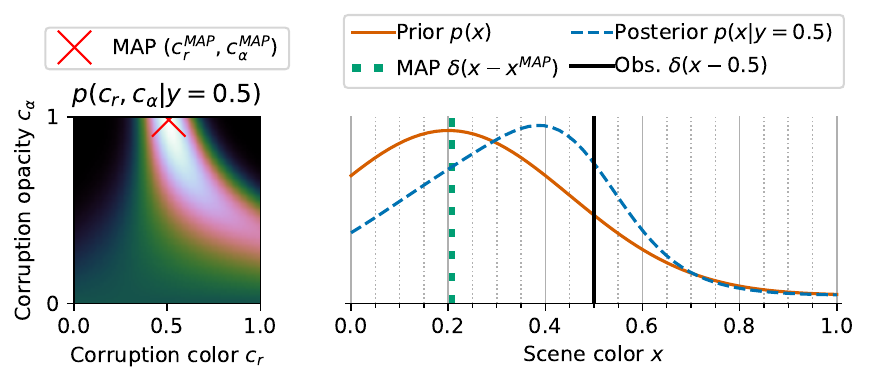}
    \vspace{-1em}
\caption{
    A toy model of full posterior inference avoiding ``billboard'' solutions of MAP.
    See main text for details.
}
    \label{fig:toy_diagram}
    \vspace{-1em}
\end{figure}

\textbf{MAP inference isn't enough}
A straightforward approach to inferring the scene $x$ would be to find the MAP estimate $(x^*, c^*)$  that maximizes $p(x)p(c)p(y | x, c)$.
However, this approach leads to ``billboard'' solutions, where the corruption $c$ ends up explaining the scene, like a billboard placed in front of the camera\footnote{This ``billboard'' need not be flat.}.
We prove this below:

\begin{proposition}
Assume $p(c) \propto 1$, for any $x$ there exists a $c$ such that $R(x, c) = y$, and $p(y\mid x, c)$ is maximized if and only if $R(x, c) = y$. Then the set of MAP solutions is
\begin{align}
    &\argmax_{x, c} p(y, x | c)p(c) = \argmax_{x, c} p(y, x | c) \nonumber\\
    &= \{x, c: y = R(x, c) \text{ and } x \in \argmax_x p(x)\}, \label{eq:argmax}
\end{align}
that is, $x$ is the maximum \emph{a-priori} scene that renders exactly to $y$, either because $c$ covers it completely or because the uncovered parts happen to render to $y$.
\end{proposition}

\begin{proof}
Let the maximum likelihood, prior, and joint probability values be $L^* := \max_{x, c} p(y | x, c)$, $P^* = \max_x p(x)$ and $V^* := \max_{x, c} p(y, x | c) = \max_x (p(x)\max_c p(y\mid x, c)) = \max_x p(x)L^* = L^*P^*$.
All elements in set \eqref{eq:argmax} attain the value $p(y | x, c)p(x)=V^*$, hence they belong to the argmax set.
Any $(x, c)$ not in set \eqref{eq:argmax} attains a value of $p(y| x, c)p(x)$ smaller than $V^*$: if $R(x, c)\neq y$, the likelihood is smaller than $L^*$; if $x \notin \argmax_x p(x)$, the prior density is smaller than $P^*$; in either case the joint is less than $L^*P^*$.
Hence set \eqref{eq:argmax} contains all argmax values.
\end{proof}

\textbf{Full posterior inference as an effective alternative}
In RIG, we perform full posterior inference to obtain the underlying scene $x, c \sim p(x, c | y) \propto p(x)p(c)p(y | x, c)$ which empirically circumvents the billboard solution (\Cref{sec:results}).
Intuitively, this can be seen as an instance of the mode not being the same as the typical set.
The region around the mode where the corruption completely covers the scene has high density but low volume---there aren't many corruptions that render exactly to the observed image.
On the other hand, the posterior takes into account both density and volume, concentrating on regions with high probability mass---there are many non-billboard corruptions that together with a correct scene render to the observed image, although each such solution may have low density.

We illustrate this in a toy example (\Cref{fig:toy_diagram}) where the scene is defined by a single opaque pixel with a one-channel color $x \sim \mathcal{N}(0.2, 0.25)$ (the distribution is truncated and normalized on $[0, 1]$) and a single semi-transparent floater corruption $c := (c_r, c_\alpha)$ between it and the camera with color $c_r \sim \mathcal{U}(0, 1)$ and opacity $c_\alpha \sim \mathcal{U}(0, 1)$.
We observe the color $y$ under a likelihood $\mathcal N(\mu(z, c), 0.1^2)$ whose mean is the result of alpha blending the scene and the corruption elements.
The MAP solution that maximizes $p(c_r, c_\alpha, x | y=0.5)$ assigns the corruption opacity $c_\alpha = 1$ (\Cref{fig:toy_diagram}, bottom-left, red cross), completely obscuring the scene.
This causes the MAP solution for $x$ to revert to the maximum a-priori value (\Cref{fig:toy_diagram}, bottom-right, dotted and solid line).
However if we perform full posterior inference (quadrature in this case), we observe high probability mass in regions where $c_\alpha < 1$ (\Cref{fig:toy_diagram}, bottom-left), where the corruption doesn't completely obscure the scene.
As a result, the posterior over the scene $x$ has more mass around the value of $y$ (\Cref{fig:toy_diagram}, bottom-right, dashed line).

\textbf{Variational inference}
We use variational inference where we optimize the evidence lower bound (ELBO) with respect to a guide distribution $q(x, c)$:
\begin{align}
    \ELBO(q) = \E_{q(x, c)}[\log p(y, x | c) - \log q(x, c)].
\end{align}
We use a mean-field parameterization where $x$ and $c$ are independent $q(x, c) = q(x)q(c)$, with each dimension of $x$ and $c$ being parameterized by a separate Gaussian mean and log standard deviation.
We optimize these parameters using stochastic gradients of the ELBO, estimated via the path derivative estimator \citep{roeder2017sticking}.
We run the optimization multiple times (8 in our experiments) and pick the run with the largest ELBO to avoid getting stuck in local optima \citep[cf. Multi-IVON from][]{shen2024variational}.

\section{Diffusion Scene Priors}
\label{sec:diffusion_scene_priors}

Denoising diffusion \citep{ho2020denoising} has emerged as a powerful alternative to normalizing flows.
While it is possible to directly replace the RealNVP used in ProbNeRF with a diffusion-based prior \citep[e.g.][]{dupont2022data}, diffusion models allow us to tractably increase the dimensionality of our latent representation.
A high dimensional latent space enables high fidelity samples and reconstructions.
We build on the Single-Stage Diffusion NeRF (SSDNeRF) framework \citep{chen2023singlestage} to train the scene prior.
SSDNeRF optimizes a set of per-training-example latents $\{x_n\}$, also known as GLO latents \citep{bojanowski2018optimizing}, the diffusion prior $p_\phi(x)$ parameterized by $\phi$, and the likelihood $p_\psi(y | x)$ parameterized by $\psi$.
See \Cref{app:priors} for additional details.

\textbf{Diffusion models}
A diffusion model is a latent variable generative model comprising a forward and a reverse process.
The forward diffusion process $q(z | x)$ starts from data $x$ and sequentially adds Gaussian noise to produce a set of latent variables $z := \{z_t; t \in [0, 1]\}$ where $t = 0$ has least noise and $t = 1$ has most noise.
The forward process is defined through its marginals $q(z_t | x) = \mathcal N(z_t; \alpha_t x, \sigma_t^2 I)$ with $\alpha_t$ and $\sigma_t$ following a schedule where the signal-to-noise ratio $\alpha_t^2 / \sigma_t^2$ decreases as $t$ increases.
We use the variance-preserving schedule where $\alpha_t^2 = 1 - \sigma_t^2$.

The reverse diffusion process $p_\phi(x, z)$, parameterized by $\phi$, starts from $p(z_1) = \mathcal N(z_1; 0, I)$ and removes noise at each step via $p_\phi(z_s | z_t)$ ($0 \leq s < t \leq 1$) to produce gradually less noisy latents until $z_0$ which is decoded into $x$ via a fixed process $p(x | z_0)$.
In practice, this is done by picking a finite number of discretization bins $T$ in $[0, 1]$.
The reverse process is trained to match $q(x)q(z | x)$ given a target data distribution $q(x)$ so that the marginal $p_\phi(x)$ ends up matching $q(x)$.
In the continuous limit ($T \to \infty$), the optimal $p_\phi(z_s | z_t)$ is a Gaussian which can be parameterized as the data-conditional forward process $q(z_s | z_t, x = x_\phi(z_t; t))$ where the data is given by a denoising model $x_\phi(z_t; t)$ that predicts $x$ given $z_t$ at time $t$.
In practice, we parameterize the noise content of $z_t$ via $\epsilon_\phi(z_t; t)$ such that $x_\phi(z_t; t) = (z_t - \sigma_t \epsilon_\phi(z_t; t)) / \alpha_t$ and train it by minimizing the loss
\begin{align}
    \mathcal L_{\text{diff}}(\phi, x) = \E_{t \sim U(0, 1), \epsilon \sim \mathcal N(0, I)}\left[w(t) \| \epsilon_\phi(z_t; t) - \epsilon \|^2\right],
\end{align}
where $w(t)$ is a weight schedule, often set to one.
We pick $w(t)$ so that the negative loss is the evidence lower bound $-\mathcal L_{\text{diff}}(\phi, x) \leq \log p_\theta(x)$ \citep{kingma2021variational}.

\textbf{NeRF representations}
The diffusion prior is defined over a NeRF representation $x$.
We experiment with triplanes, the representation used in SSDNeRF, as well as a set latent representation based on the scene representation transformer (SRT) \citep{sajjadi2022scene}.

While we found that the triplane representation---which uses a UNet denoiser \citep{ronneberger2015unet}---was sufficient for the ShapeNet dataset, SRT's set latents performed much better on the more complex MultiShapeNet dataset.
Since SRT doesn't learn a prior over $x$, we adopt a transformer-based denoiser that is permutation invariant based on Point-E \citep{nichol2022point}.
In our experiments, we report performance of the triplane representation for ShapeNet and SRT's set latents for MultiShapeNet.

In both cases we train a decoder, parameterized by $\psi$, mapping from the representation of a ray position and direction $(x_r, d_r)$ to a color and a density $(\gamma, \sigma)$.
For triplanes, it is a small MLP and for set latents, it is SRT's decoder transformer.
Like before, we can form a renderer $R(x)$ (omitting camera parameters $\zeta$) and the corresponding likelihood $p_\psi(y | x)$, and the corresponding variants for the case with corruption: $R(x, c)$, $p_\psi(y | x, c)$.

\textbf{Training}
We adopt SSDNeRF's training procedure where given a training dataset $\{y_n\}_{n = 1}^N \sim p(y)$, we maintain a set of per-training-example latents $\{x_n\}_{n = 1}^N$ which are co-trained with the prior and likelihood parameters $(\phi, \psi)$.
Interpreting the negative log likelihood as the reconstruction loss $\mathcal L_{\text{rec}}(y, x, \psi) = \log p_\psi(y | x)$, the training loss is
\begin{align}
    &\mathcal L(\{x_n\}, \phi, \psi) = \\
    &\E_{n \sim U(\{1, \dotsc, N\})}\left[\lambda_{\text{rec}} \mathcal L_{\text{rec}}(y_n, x_n, \psi) + \lambda_{\text{diff}} \mathcal L_\text{diff}(\phi, x_n)\right] \nonumber
\end{align}
where $\lambda_{\text{rec}}, \lambda_{\text{diff}} > 0$ weighting factors.
The negative loss, $-\mathcal L(\{x_n\}, \phi, \psi) \leq \E_{p(y)}\left[\log p_{\theta, \psi}(y)\right]$, can be interpreted as an ELBO where the guide distribution is a delta mass on $x_n$ if $y = y_n$, $q_{\{x_n\}}(x | y_n) = \delta_{x_n}(x)$, making the guide's entropy term in ELBO vanish; the guide is well-defined during training.
Once trained, we can discard the latents and use $p_\phi(x) p_\psi(y | x)$ as the scene prior. See \Cref{app:training} for additional details.

\section{Diffusion Conditioning with Auxiliary Latents}
\label{sec:diffusion_conditioning}

Given a diffusion model, we can sample from $p(x)$ by iteratively denoising $z_1 \sim \mathcal N(0, I)$ until $z_0$ and decoding $x \sim p(x | z_0)$, however it is not easy to evaluate the density $p(x)$ due to the required marginalization of $z$.
This makes it difficult to infer $p(x | y) \propto p(x)p(y | x)$ as well as $p(x, c | y) \propto p(x)p(c)p(y | x, c)$.
We review reconstruction-guidance as a method for solving the former and propose reconstruction-guidance with auxiliary latents (ReGAL) for solving the latter.
\begin{algorithm}[!tb]
   \caption{Reconstruction-guidance diffusion conditioning with auxiliary latents (ReGAL)}
   \label{alg:regal}
\begin{algorithmic}[1]
   \STATE {\bfseries Input:} Diffusion prior denoiser $\epsilon_\phi(z_t; t)$, prior over the auxiliary latent $p(c | x)$, likelihood $p(y | x, c)$, Langevin step size $\delta$, number of discretization bins $T$ and a discretization schedule $s(i) = (i - 1) / T, t(i) = i / T$, observation $y$.
   \STATE {\bfseries Output:} Approximate sample $x, c$ from $p(x, c | y) \propto p(x)p(c | x)p(y | x, c)$.
   \STATE Initialize $z_1 \sim p(z_1) = \mathcal N(z_1; 0, 1)$ and $c_1 \sim p(c | z_1)$.
   \FOR{$i = T, \dotsc, 1$}
   \STATE Sample $z_{s(i)} \sim \hat p(z_{s(i)} | z_{t(i)}, y)$.
   \STATE Sample $c_{s(i)} \sim \hat p(c_{s(i)} | z_{s(i)}, c_{t(i)}, y; \delta)$.
   \ENDFOR
   \STATE Sample $x \sim p(x | z_0), c \sim \hat p(c | x, c_0, y; \delta)$.
   \STATE {\bfseries Return:} $x, c$.
\end{algorithmic}
\end{algorithm}

\textbf{Reconstruction-guidance}
Reconstruction-guidance conditioning \citep{ho2022video} modifies the unconditional sampling process by sampling $z_s \sim p(z_s | z_t, y)$ instead of from $p(z_s | z_t)$ at each discretization step.
Since $p(z_s | z_t, y)$ is intractable, we approximate it by augmenting the score function in $p(z_s | z_t)$ by the gradient of the log likelihood.
The score of the diffusion latent variable at time $t$ can be estimated as $\nabla_{z_t} \log q(z_t) \approx s_\phi(z_t; t) = -\epsilon_\phi(z_t; t) / \sigma_t$ and the reverse diffusion process can be written in terms of this score as
\begin{align}
    p(z_s | z_t) = \mathcal N\left(z_s; \frac{z_t}{\alpha_{t | s}}  + \frac{\sigma_{t | s}^2}{\alpha_{t | s}} s_\phi(z_t; t), \frac{\sigma_{t | s}^2 \sigma_s^2}{\sigma_t^2} I\right) \label{eq:score_based_reverse_difusion}
\end{align}
where $\alpha_{t | s} = \alpha_t / \alpha_s$ and $\sigma_{t | s}^2 = \sigma_t^2 - \alpha_{t | s}^2 \sigma_s^2$ \citep{kingma2021variational}.
In reconstruction-guidance conditioning, we replace the score term in \eqref{eq:score_based_reverse_difusion} by a ``reconstruction-guided'' score which approximates the posterior score $\nabla_{z_t} \log p(z_t | y)$ as
\begin{align}
    \tilde s_\phi(z_t, y; t) := s_\phi(z_t; t) + \nabla_{z_t} \log \tilde p(y | z_t) \label{eq:reconstruction_guided_score}
\end{align}
where the intractable likelihood $p(y | z_t)$ is approximated as $\tilde p(y | z_t) := p(y | x = x_\phi(z_t; t))$, using $x$'s ``reconstruction'':
\begin{align}
    \hat p(z_s | z_t, y) := \mathcal N\left(z_s; \frac{z_t}{\alpha_{t | s}}  + \frac{\sigma_{t | s}^2}{\alpha_{t | s}} \tilde s_\phi(z_t, y; t), \frac{\sigma_{t | s}^2 \sigma_s^2}{\sigma_t^2} I\right). \label{eq:reconstruction_guidance_conditional}
\end{align}

\textbf{Reconstruction-guidance with auxiliary latents (ReGAL)}
In RIG, we have a diffusion prior over the scene $p(x)$ as well as a prior over the corruptions $p(c)$.
This is an instance of a more general setting where $c$ is an auxiliary latent variable with a prior $p(c | x)$.
We propose ReGAL (\Cref{alg:regal}) for approximating the posterior $p(c, x | y) \propto p(x)p(c | x)p(y | x, c)$ for a given likelihood $p(y | x, c)$, such as the scene-and-corruption renderer in RIG.

ReGAL alternates between sampling from the reconstruction-guidance conditional $\hat p(z_s | z_t, y)$ \eqref{eq:reconstruction_guidance_conditional} and a Langevin update \citep{welling2011bayesian} on $c$:
\begin{align}
    \hat p(c_s | z_t, c_t, y; \delta) := \mathcal N\left(c_s; c_t + \frac{\delta}{2} \tilde s_\phi^c(z_t, c_t, y; t), \delta I\right) \label{eq:langevin_step}
\end{align}
where $\delta$ is a step size and the score of the target distribution $p(c | x, y)$ is approximated as
\begin{align}
    &\tilde s_\phi^c(z_t, c_t, y; t) := \label{eq:regal_score}\\
    &\nabla_{c_t} \log p(c_t | x = x_\phi(z_t; t)) + \nabla_{c_t} \log p(y | x = x_\phi(z_t; t), c_t) \nonumber 
\end{align}
Running these Langevin updates results in a sample from the posterior $p(c | x, y)$.
Like reconstruction-guidance conditioning, ReGAL does not converge to the target distribution of interest since it relies on score function estimation and marginalization by denoising.
In most cases, we find that ReGAL produces compelling results for RIG tasks as it intuitively moves $(x, c)$ towards a region with high likelihood $p(y | x, c)$ via the likelihood score.

In some experiments, we find that this basic version of ReGAL is not enough as it samples in an ``open-loop'', without re-checking the samples against the posterior of interest.
To address this problem, we re-interpret ReGAL as an importance sampling algorithm which produces $K$ samples $\{(x_k, c_k), w_k\}_{k = 1}^K$.
An importance weight $w_k$ is a ratio between the unnormalized target distribution---in this case formed by multiplying the diffusion prior over all the noised variables times the likelihood $p(z) p(x | z)p(c | x)p(y | x, c)$---and the ReGAL proposal.
These importance weights are used to check the quality of proposed samples against the posterior of interest.
When $K > 1$, ReGAL has convergence guarantees consistent with the classical importance sampling results.
The weighted set of $K$ samples can be used to estimate the posterior expectation of any test function $f$, $\E_{p(x, c | y)}[f(x, c)]$ as $\sum_{k = 1}^K w_k f(x_k, c_k) / \sum_{k = 1}^K w_k$.
This estimator is consistent: it converges to the posterior expectation with increasing $K$ \citep{owen2013monte}.
We can also resample from the weighted empirical distribution $\sum_{k = 1}^K w_k \delta_{(x_k, c_k)}(x, c) / \sum_{k = 1}^K w_k$ to get an unweighted sample.
See \Cref{app:is_regal} and \Cref{alg:is_regal} for an extended description of how to incorporate the $K$ hyperparameter into ReGAL.

\begin{figure*}[!tb]
  \centering
    \includegraphics[scale=0.55]{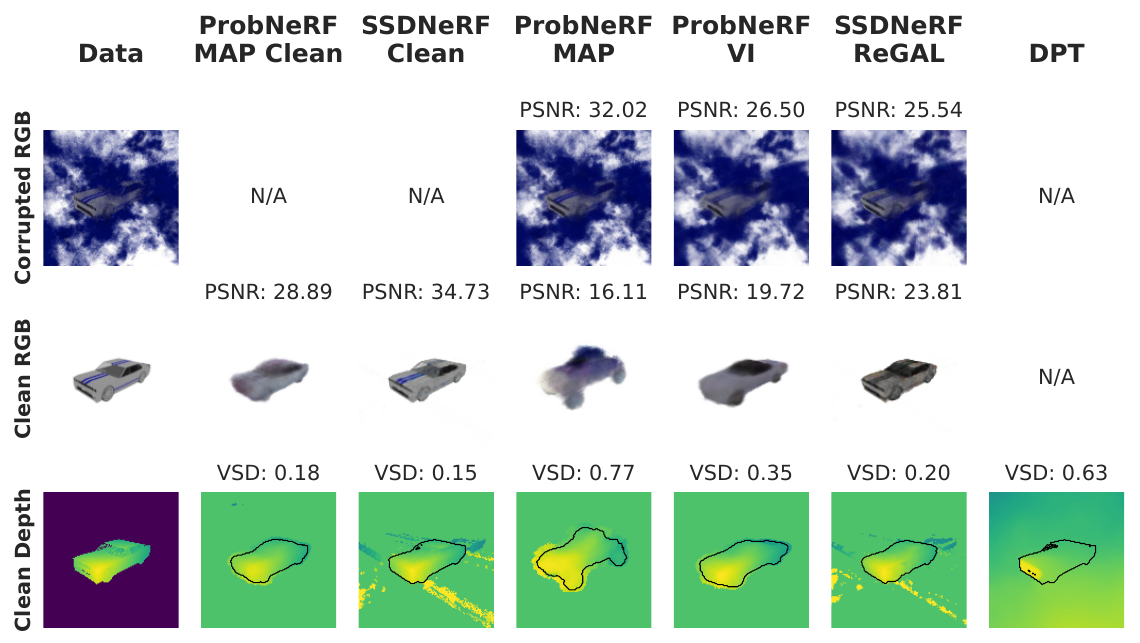}
    \includegraphics[scale=0.55]{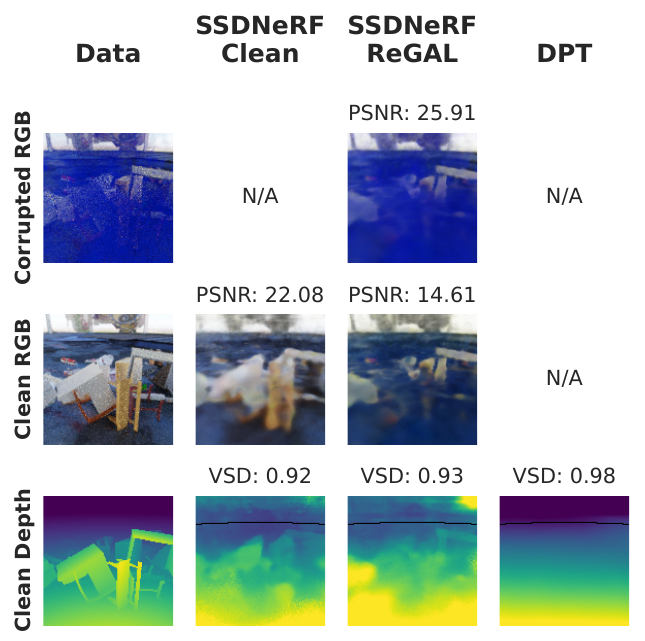}
    \vspace{-1em}
\caption{
    Example decorruptions for the cloud corruption. The \emph{Clean} columns are conditioned on the \emph{Clean RGB} data, while the rest are conditioned on the \emph{Corrupted RGB} data.
    For ShapeNet, the black outlines on the depth images are the predicted masks, except in the case of DPT where ground truth mask is used. For MultiShapeNet, ground truth mask is used for all.
}
    \vspace{-1em}
    \label{fig:recons}
\end{figure*}
\begin{figure}[!tb]
  \centering
    \includegraphics[scale=0.55]{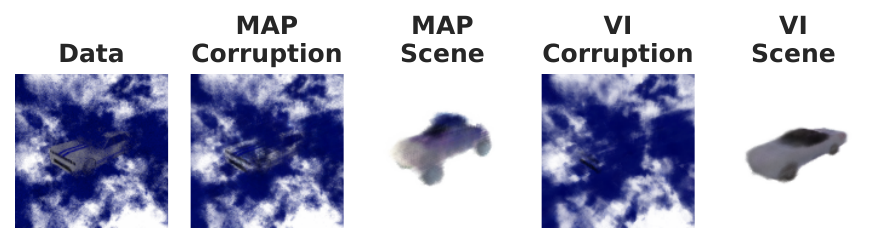}
    \vspace{-1em}
\caption{
    MAP solution uses the corruption NeRF to explain the observation more than the VI solution.
}
\vspace{-2em}
    \label{fig:billboard}
\end{figure}

\section{Related Work}

\textbf{Robust Monocular Depth Estimation}
Most robust depth estimation algorithms are trained via regression on a diverse set of datasets and are not expected to generalize far from the training distribution \citep{zhao2022unsupervised, gasperini2023robust, zhu2023ecdepth, saunders2023selfsupervised}.
The primary challenge of these works is how to generate a sufficient amount of data to train the regression models across enough conditions.
In comparison, our method requires only a coarse model of the corruptions ahead of time.

\textbf{Robust Multi-View Reconstruction}
In the field of NeRF training the problem of robust reconstruction is well explored.
Methods that use robust losses \citep{sabour2023robustnerf}, 3D regularization \citep{wynn2023diffusionerf, warburg2023nerfbusters} and uncertainty estimation \citep{goli2023bayes} have produced impressive results.
Most of those methods focus on the removal of floaters, and while certain corruptions considered in our work (e.g. rain/clouds) are similar to floaters, it is not obvious how to extend these robustness techniques to other corruption types. 

\textbf{Analysis by Synthesis} Using generative models in visual perception is rare in modern systems in part due to performance issues.
A number of systems combine NeRFs with structured generative models with nuisance variables \citep{yuan2020star, verbin2021refnerf, park2023camp, ost2021neural}.
We follow this line of work and extend it with priors over NeRFs and probabilistic inference over the joint distribution of scenes and auxiliary parameters.
While it is not our main focus, the posterior samples can be used to quantify scene uncertainty \citep{guo2020bayesian,rodriguez2023umat,hoffman2023probnerf}.

\textbf{Diffusion conditioning}
To the best of our knowledge, we are the first to consider the problem of diffusion conditioning with auxiliary latents: conditioning models where a diffusion model is used as a prior only over a subset of latent variables.
Our proposed method, ReGAL, bears similarities to existing diffusion conditioning methods.
Unlike classifier-free guidance \citep{ho2021classifier} and classifier-guided diffusion \citep{dhariwal2021diffusion} which require retraining for each new likelihood, we focus on the case when this is undesirable or even impossible.
A common approach for addressing this problem is reconstruction-guidance \citep{ho2022video} which ReGAL builds upon.
\citet{wu2023practical} propose an sequential Monte Carlo (SMC)-based generalization of reconstruction-guidance which comes with convergence guarantees like ReGAL with $K > 1$.
However, it isn't applicable to models with auxiliary latents.
While it is possible to design an SMC version of ReGAL (\Cref{app:smc_regal}), it didn't improve performance metrics in our domain.

\begin{figure*}[!tb]
  \centering
    \includegraphics[scale=0.55]{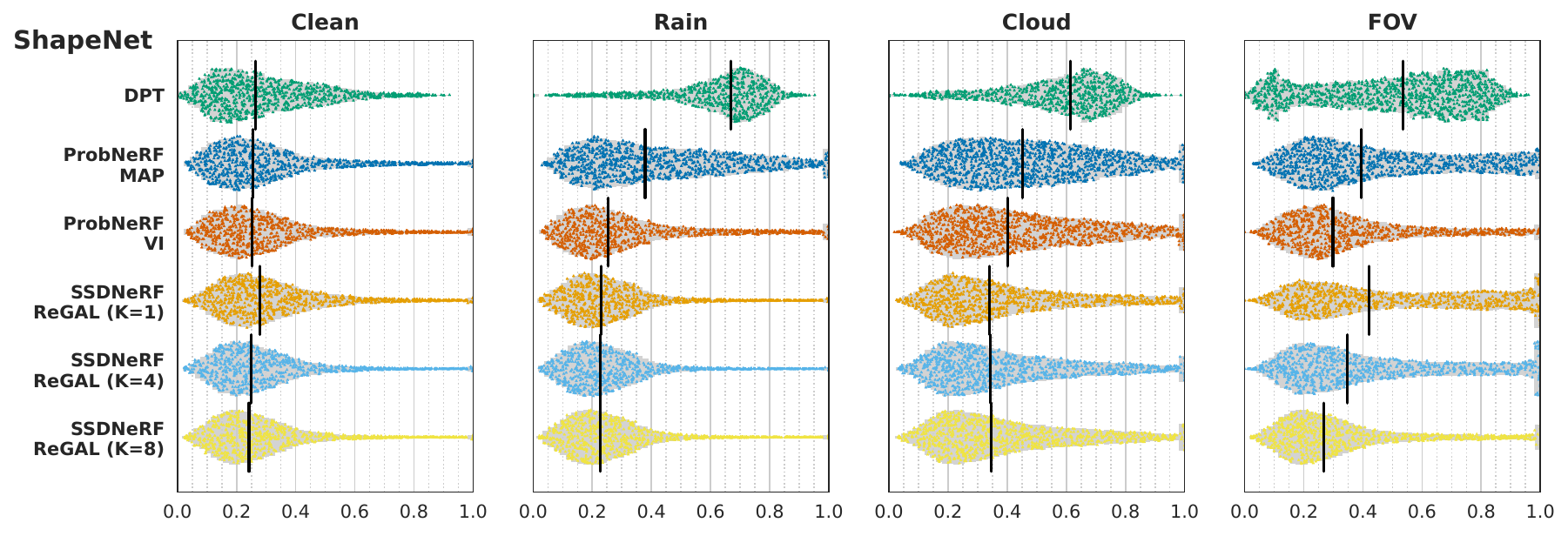}
    \includegraphics[scale=0.55]{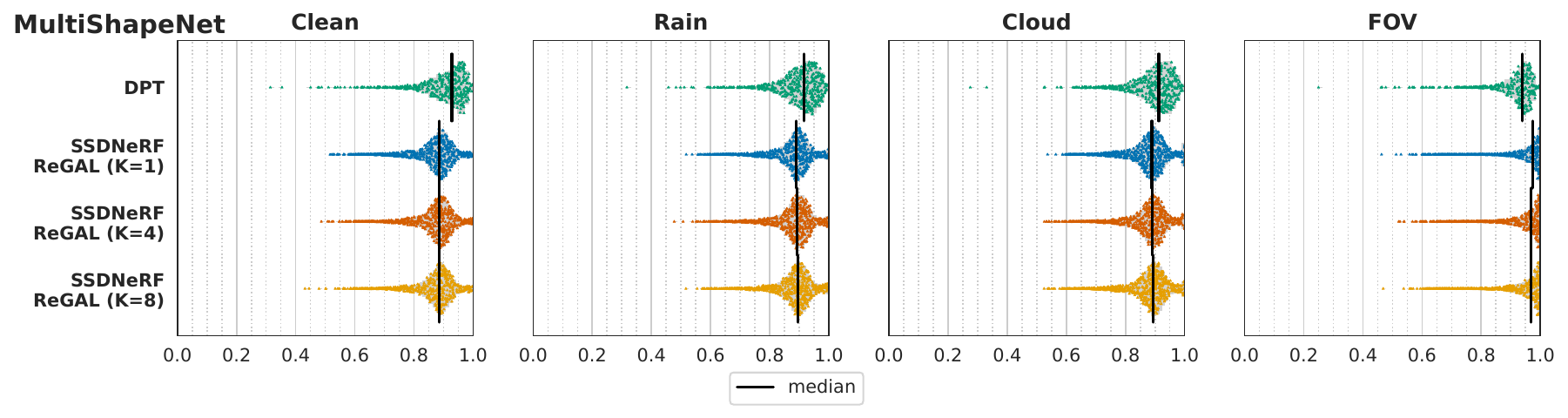}
    \vspace{-1.5em}
\caption{
    VSD($\leftarrow$) histograms across all conditions. Lower is better.
}
    \vspace{-1.5em}
    \label{fig:vsd_polkagrams}
\end{figure*}

\section{Experiments}

The core benchmark we explore is monocular depth estimation in presence of corruptions.
Monocular depth estimation is useful in context of autonomous car navigation \citep{jing2022depth} and robotics \citep{zhou2019does}.
In this task, the models are given a single RGB image of a scene and must predict a depth image, representing, for each pixel, the distance of the corresponding object surface element from the camera.
We also evaluate the quality of RGB reconstructions from multiple test views given one view of a scene with corruptions.
The source code for many of these experiments is available at \url{https://github.com/tensorflow/probability/tree/main/discussion/robust_inverse_graphics}.

\textbf{Datasets} We evaluate our method on two datasets.
For the first dataset, we use the cars category from ShapeNet \citep{chang2015shapenet}.
The dataset consists of 3486 cars, where 3137 are used for training and the remaining 349 for evaluation.
The cars are placed at the origin and have a canonical orientation, but often vary in their scale.
For each training scene, we generate 50 camera positions randomly from a unit sphere, all oriented towards the scene origin.
For testing, the camera positions are generated as $\{\cos(k \pi / 8), \pi / 8, \sin(k \pi / 8)\} | k \in [0, 16)$.

For the second dataset, we use the MultiShapeNet (MSN) dataset \citep{sajjadi2022scene}.
This dataset consists of objects from a number of categories from ShapeNet placed in a square region (side length 5) on a randomly textured ground plane, and lit by a randomly chosen HDRI dome.
Aside from adding our corruptions, we follow \citet{sajjadi2022scene} for the scene and camera positions.
We generate 10000 training scenes and 300 evaluation scenes.
The training set scenes have 100 camera positions, and testing ones have 10.

For the rain and fog corruptions we use the Blender-provided particle and volumetric effects.
For the FOV, we randomly sample the FOV from $[\pi / 4, 3 \pi / 4]$.
We use Kubric \citep{greff2022kubric}, which uses Blender to render and provide ground truth RGB, depth and segmentation images.
All images have resolution of $128 \times 128$.

\textbf{Metrics} For depth estimation, we use the visible surface discrepancy (VSD) metric \citep{hodan2016evaluation}, defined as:
\begin{align*}
    &VSD(d, M, \hat{d}, \hat{M}) = 1 - \\
    &\frac{1}{|M \cup \hat{M}|} \sum_{p \in M \cup \hat{M}} \mathds{1}(p \in M \cap \hat{M} \land |d(p) - \hat{d}(p)| < \tau),
\end{align*}
where $d$ and $\hat{d}$ are the true and predicted depth images, $M$ and $\hat{M}$ are the true and predicted object masks, $p$ is a pixel position, $\mathds{1}(\cdot)$ is an indicator function and $\tau$ is an accuracy threshold.
We use $\tau = 0.05$ for ShapeNet and $\tau = 0.1$ for MSN.
For NeRFs we define the mask as the region where the opacity of the object is greater than $0.5$.
Predicted depth is defined as the 95th percentile distance of the ray scattering.
VSD is a convenient metric because we are solely interested in recovering the depth maps corresponding to the objects, rather than the background, and VSD penalizes making mistakes about the object silhouette.
For MSN, we evaluate the depth predictions on the parts of the image that are explained by the ground plane and the objects on it.
This is operationally defined as regions of the image where depth is less than 30 units.

For images with volumetric corruptions, we compute VSD relative to the uncorrupted image (and omit the corruption NeRF when rendering the reconstruction).
For RGB image comparisons we use the PSNR metric, defined as $-10 \log_{10} \mathrm{MSE}$, where MSE is mean-squared-error across pixels and channels.

\begin{table*}[!htb]

\centering

\begin{tabular}{lllll}
\multicolumn{5}{c}{ShapeNet} \\
\toprule
Model & Clean & Rain & Cloud & FOV \\
\midrule
DPT & $0.30 \pm 0.007$ & $0.63 \pm 0.0065$ & $0.58 \pm 0.0071$ & $0.49 \pm 0.01$ \\
ProbNeRF
MAP & $0.30 \pm 0.0034$ & $0.43 \pm 0.0044$ & $0.48 \pm 0.0042$ & $0.46 \pm 0.0047$ \\
ProbNeRF
VI & $0.30 \pm 0.0034$ & $0.32 \pm 0.0039$ & $0.45 \pm 0.0043$ & $0.36 \pm 0.004$ \\
SSDNeRF
ReGAL (K=1) & $0.32 \pm 0.0032$ & $0.26 \pm 0.0027$ & $\mathbf{0.40 \pm 0.0043}$ & $0.49 \pm 0.0052$ \\
SSDNeRF
ReGAL (K=4) & $0.28 \pm 0.0029$ & $\mathbf{0.25 \pm 0.0027}$ & $\mathbf{0.40 \pm 0.0042}$ & $0.43 \pm 0.0049$ \\
SSDNeRF
ReGAL (K=8) & $\mathbf{0.27 \pm 0.0029}$ & $\mathbf{0.25 \pm 0.0027}$ & $0.41 \pm 0.0042$ & $\mathbf{0.34 \pm 0.0041}$ \\
\bottomrule
\multicolumn{5}{c}{MultiShapeNet} \\
\toprule
Model & Clean & Rain & Cloud & FOV \\
\midrule
DPT & $0.91 \pm 0.0037$ & $0.90 \pm 0.0035$ & $0.90 \pm 0.0035$ & $\mathbf{0.92 \pm 0.0031}$ \\
SSDNeRF
ReGAL (K=1) & $\mathbf{0.87 \pm 0.0015}$ & $\mathbf{0.88 \pm 0.0015}$ & $\mathbf{0.88 \pm 0.0015}$ & $0.95 \pm 0.0013$ \\
SSDNeRF
ReGAL (K=4) & $\mathbf{0.87 \pm 0.0015}$ & $\mathbf{0.88 \pm 0.0015}$ & $\mathbf{0.88 \pm 0.0015}$ & $0.95 \pm 0.0013$ \\
SSDNeRF
ReGAL (K=8) & $0.88 \pm 0.0015$ & $0.89 \pm 0.0014$ & $\mathbf{0.88 \pm 0.0015}$ & $0.95 \pm 0.0013$ \\
\bottomrule

\end{tabular}

\caption{
Mean VSD values across all conditions. Lower is better. The confidence intervals are 3 $\times$ SEM computed from a five runs via bootstrap.
}
\label{tab:vsd}
\vspace{-1em}
\end{table*}

\textbf{Baseline} We use DPT \citep{ranftl2021vision} as a baseline.
DPT is a powerful generic depth estimator, which uses a feedforward model that outputs the depth up to an affine transformation.
DPT is trained on a variety of indoor and outdoor scenes with ground truth depths.
When computing the VSD metric for this model, we use the ground-truth mask $M$ as its predicted $\hat{M}$ and determine the affine transformation of its depth by minimizing squared error between the model's predictions and ground truth depths in the masked region.
This gives this baseline a strong advantage in that it does not need to solve the scale-distance ambiguity.

\subsection{MAP vs full posterior inference}
\label{sec:results}

\Cref{fig:vsd_polkagrams} and \Cref{tab:vsd} summarize the depth reconstruction results.
For each algorithm and dataset, we aggregate the metrics across all views and all test scenes.
While for every algorithm there is considerable variation in the performance of the algorithms, in aggregate we see that probabilistic inference (ProbNeRF VI and SSDNeRF ReGAL) outperform the point estimates and the regression baseline on the corrupted scenes for ShapeNet.
For MSN, our prior does not model the data distribution with enough fidelity to resolve the fine details of objects (see \Cref{fig:msn_prior_samples}), and therefore does not decisively outperform the relatively coarse estimate that DPT provides.
For FOV estimation in particular, small errors in the estimated FOV parameter lead to large changes in the predicted depths, which is something DPT does not need to contend with due to the calibration procedure we use to evaluate its predictions.
Improving the quality of the prior would be a natural area of future improvement.

In addition to VSD, we also compute the PSNR metric for all conditions.
As with VSD, we remove the corruption from both the data and the reconstruction. \Cref{fig:psnr_polkagrams} and \Cref{tab:psnr} show that ReGAL outperforms the other conditions.
We do not have a regression baseline for this task.

For ReGAL, we vary the number of importance samples $K$. The decorruption quality is weakly dependent on this for most settings, but is is most noticeably beneficial for the FOV estimation task for ShapeNet.

\Cref{fig:recons} shows an illustrative scene decorruption result for the cloud corruption and several of the compared methods.
When possible, for each method we show the full reconstruction (scene and corruption), just the scene reconstruction and the depth image for the scene.
For ProbNeRF, we see that the VI full inference has a more faithful reconstruction than a MAP point estimate, due to MAP's over-reliance on the corruption parameters to explain most of the image, as can be also be seen in \Cref{fig:billboard}.
The more powerful prior used for SSDNeRF model produces the best reconstructions on clean images (``SSDNeRF Clean'' column of \Cref{fig:recons}) and, when used with ReGAL, the most accurate depth image.
For MSN, we do not have a point estimate baseline, as ProbNeRF is not expressive enough to represent MSN scenes.

Due to its design and training dataset, DPT cannot separately extract the depth image for the obscured sections of the car, and is penalized accordingly.
A more powerful depth estimator that supports multi-modal predictions \citep{saxena2023monocular, duan2023diffusion, ke2023repurposing}, if trained on corrupted images, could, in principle, match our method's performance.
Probabilistic inference in our model makes sharp predictions without requiring examples of corruptions.

\begin{figure*}[!htb]
  \centering
    \includegraphics[scale=0.55]{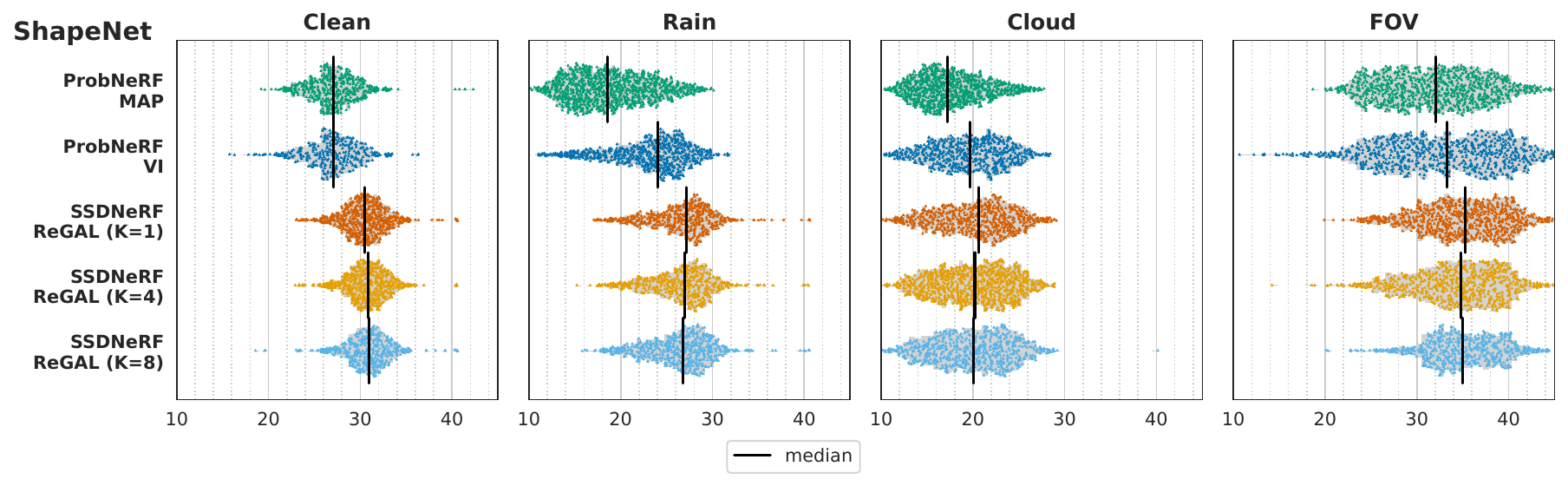}
    \includegraphics[scale=0.55]{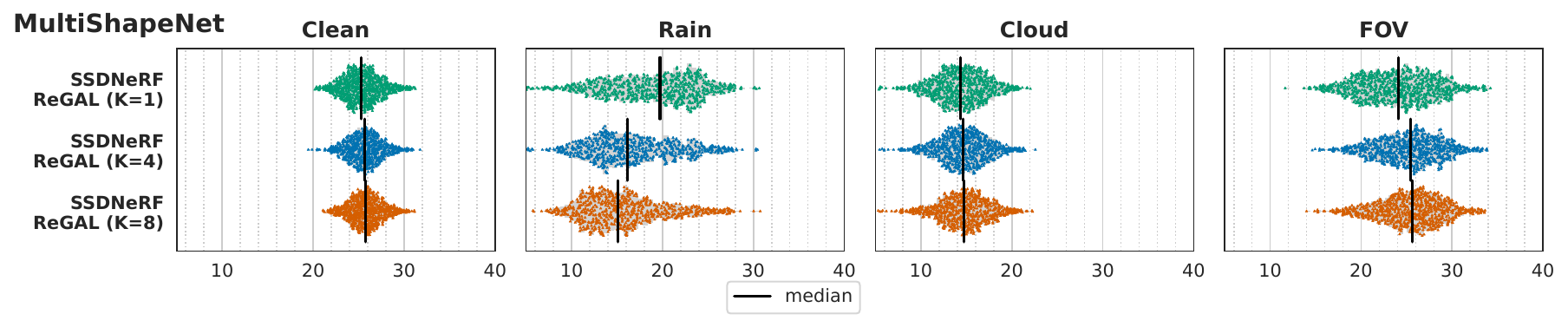}
    \vspace{-1em}
\caption{
    PSNR($\rightarrow$) histograms across all conditions.
    Higher is better.
}
    \label{fig:psnr_polkagrams}
    \vspace{-1em}
\end{figure*}

\begin{table*}[!htb]

\centering

\begin{tabular}{lllll}
\multicolumn{5}{c}{ShapeNet} \\
\toprule
Model & Clean & Rain & Cloud & FOV \\
\midrule
ProbNeRF
MAP & $27.08 \pm 0.17$ & $18.96 \pm 0.3$ & $17.68 \pm 0.25$ & $31.88 \pm 0.38$ \\
ProbNeRF
VI & $27.11 \pm 0.2$ & $23.22 \pm 0.29$ & $19.54 \pm 0.27$ & $32.97 \pm 0.47$ \\
SSDNeRF
ReGAL (K=1) & $30.49 \pm 0.16$ & $\mathbf{26.58 \pm 0.22}$ & $19.89 \pm 0.31$ & $35.08 \pm 0.32$ \\
SSDNeRF
ReGAL (K=4) & $30.80 \pm 0.15$ & $26.47 \pm 0.22$ & $\mathbf{19.95 \pm 0.29}$ & $34.44 \pm 0.33$ \\
SSDNeRF
ReGAL (K=8) & $\mathbf{30.87 \pm 0.15}$ & $26.32 \pm 0.23$ & $19.83 \pm 0.29$ & $\mathbf{35.22 \pm 0.27}$ \\
\bottomrule
\multicolumn{5}{c}{MultiShapeNet} \\
\toprule
Model & Clean & Rain & Cloud & FOV \\
\midrule
SSDNeRF
ReGAL (K=1) & $25.37 \pm 0.15$ & $\mathbf{19.01 \pm 0.38}$ & $14.40 \pm 0.21$ & $23.99 \pm 0.31$ \\
SSDNeRF
ReGAL (K=4) & $25.68 \pm 0.13$ & $16.74 \pm 0.35$ & $14.65 \pm 0.2$ & $25.19 \pm 0.26$ \\
SSDNeRF
ReGAL (K=8) & $\mathbf{25.79 \pm 0.13}$ & $15.65 \pm 0.34$ & $\mathbf{14.71 \pm 0.2}$ & $\mathbf{25.39 \pm 0.26}$ \\
\bottomrule
\end{tabular}

\caption{
Mean PSNR values across all conditions. Higher is better. The confidence intervals are 3 $\times$ SEM computed from a five runs via bootstrap. Note that the PSNR values for the FOV dataset are inflated because of the white background which takes up a higher fraction of the image for larger fields-of-view values.
}
\label{tab:psnr}
\vspace{-1em}
\end{table*}

\section{Discussion}
We have presented a general strategy in attaining robustness for monocular depth estimation using probabilistic inference using scene priors.
For diffusion scene priors, we develop a general-purpose diffusion conditioning approach, ReGAL, which when applied to our domain produces the best performance in our experiments.

There are two key limitations and areas of improvement for our work.
The reliance on a 3D prior makes it difficult to apply our method to robust perception of real world scenes, as it requires a 3D prior over real world scenes (see \Cref{app:prior_samples} for prior samples from models used in the experiments).
While a lot of progress has been made in unconditional 3D models \citep{bautista2022gaudi, chen2023singlestage, chai2023persistent, kim2023nfldm}, their fidelity and breadth of coverage of real world scenes remains small.

A second limitation is that of inference speed.
Performing probabilistic inference for every image is prohibitively slow for the majority of applications.
For diffusion models, there exist approaches to accelerate sampling \citep[e.g.][]{song2023consistency}.
Real-time performance, however, will require amortized inference, likely by training a parameterized encoder.
The challenge of this approach is how to attain this speedup without losing the flexibility and robustness of our method.

\section*{Acknowledgements}

Thanks to Jyh-Jing Hwang for his helpful comments and Mehdi S. M. Sajjadi for his helpful comments and help with setting up the scene representation transformer model.

\section*{Impact Statement}

This paper presents work whose goal is to advance the field of Machine Learning. There are many potential societal consequences of our work, none which we feel must be specifically highlighted here.

\bibliography{main}
\bibliographystyle{icml2024}

\newpage
\appendix
\onecolumn

\section{NeRF Rendering Details}

\label{app:nerf_rendering}

We follow \citet{barron2021mipnerf} and use hierarchical sampling.
Given a ray $(x_r, d_r)$ we first query the proposal NeRF using 48 regularly spaced points between near $n$ and far $f$ values. For ShapeNet, $n=0.2, f=1.5$. For MSN, $n=0.2, f=30$.
The density values evaluated at those samples are used to form a proposal distribution, and 48 samples are taken from the final NeRF, and to compute the rendered color and depth.
We use the same NeRF network for both steps, and rely on integrated position encoding to inform the neural networks of the level of detail they should represent.
For grid representations (triplanes), we use multisampling by deterministically generating sigma points from the Gaussians associated with each ray sample using the unscented transformation from \citet{menegaz2011new}.
This a simplification of the approach used in \citet{barron2023zipnerf}, as the (relatively) low fidelity of our NeRFs does not require the full anti-aliasing machinery of that work.

\section{ReGAL with Importance Sampling}
\label{app:is_regal}

The basic variant of ReGAL (\Cref{sec:diffusion_conditioning}, \Cref{alg:regal}) can be viewed as a one-sample importance sampler that samples from an extended-space target density
\begin{align}
    \gamma(z_{[0, 1]}, c_{[0, 1]}, x, c) := \underbrace{p_{\text{diff}}(z_1) \left(\prod_{i = T, \dotsc, 1} p_{\text{diff}}(z_{s(i)} | z_{t(i)}) \right) p_{\text{diff}}(x | z_0)}_{\text{Probability of denoising diffusion}} \underbrace{\left(\prod_{i = T, \dotsc, 0} p_{\text{aux}}(c_{t(i)} | z_{t(i)})\right) p_{\text{aux}}(c | x)}_{\text{Probability of auxiliary latents}} \underbrace{p(y | x, c)}_{\text{Likelihood}} \label{eq:regal_is_target}
\end{align}
where we discretize the unit time interval into $T$ bins with $s(i) = (i - 1) / T$ and $t(i) = i / T$, and define the intermediate diffusion latent variables as $z_{[0, 1]} := (z_{t(i)})_{i = 0}^T$.

To aid clarity, we add subscripts to the density terms to indicate whether they come from the diffusion model (``diff'') or the prior over the auxiliary latent (``aux'').
The densities that come from the diffusion model are (i) the initial diffusion noise $p_{\text{diff}}(z_1) = \mathcal N(z_1; 0, I)$, (ii) the reverse denoising step $p_{\text{diff}}(z_{s(i)} | z_{t(i)})$ and (iii) the decoder $p_{\text{diff}}(x | z_0)$.
The prior over the auxiliary latent $p_{\text{aux}}(c | z)$ is given.
However, we introduce copies of the auxiliary latent variable $c_{[0, 1]} := (c_{t(i)})_{i = 0}^T$ (corresponding to the intermediate diffusion variables $z_{[0, 1]}$) and place the auxiliary latent prior $p_{\text{aux}}(c_{t(i)} | z_{t(i)})$ over them.

Integrating out $z_{[0, 1]}, c_{[0, 1]}$ leaves us with the desired target unnormalized density $p(x)p(c | x)p(y | x, c)$.
Hence sampling from $\pi(z_{[0, 1]}, c_{[0, 1]}, x, c) \propto \gamma(z_{[0, 1]}, c_{[0, 1]}, x, c)$ and ignoring $z_{[0, 1]}, c_{[0, 1]}$ leaves us with a sample from the desired posterior $p(x, c | y) \propto p(x)p(c | x)p(y | x, c)$.

The last piece for defining an importance sampler is the proposal distribution which we set to be the ReGAL procedure in \Cref{alg:regal}, resulting in the following proposal density
\begin{align}
    q(z_{[0, 1]}, c_{[0, 1]}, x, c) := \underbrace{p_{\text{diff}}(z_1) p_{\text{aux}}(c_1 | z_1)}_{\text{Init. before the for loop}} \underbrace{\left(\prod_{i = T, \dotsc, 1} \hat p(z_{s(i)} | z_{t(i)}, y) \hat p(c_{s(i)} | z_{s(i)}, c_{t(i)}, y)\right)}_{\text{The for loop}} \underbrace{p_{\text{diff}}(x | z_0) \hat p(c | x, c_0, y)}_{\text{Sampling after the for loop}}. \label{eq:regal_is_proposal}
\end{align}

Given a budget of $K$ samples, importance sampling samples $K$ times from $z_{[0, 1]}^k, c_{[0, 1]}^k, x^k, c^k \sim q$ and assigns each sample a weight of
\begin{align}
    w_k = \frac{\gamma(z_{[0, 1]}^k, c_{[0, 1]}^k, x^k, c^k)}{q(z_{[0, 1]}^k, c_{[0, 1]}^k, x^k, c^k)}, \quad (k = 1, \dotsc, K). \label{eq:is_regal_weight}
\end{align}
The full algorithm can be seen in \Cref{alg:is_regal}.

The posterior $p(x, c | y)$ can be approximated as a weighed sum of delta masses $\sum_{k = 1}^K \frac{w_k}{\sum_{i = 1}^K w_i} \delta_{(x^k, c^k)}(x, c)$ and the posterior expectation of any test function $f$ can be estimated as $\E_{p(x, c | y)}[f(x, c)] \approx \sum_{k = 1}^K \frac{w_k}{\sum_{i = 1}^K w_i} f(x^k, c^k)$.
Note that while ReGAL (\Cref{alg:regal}) isn't guaranteed to converge to the desired posterior, this importance sampling version of it is---the posterior expectation estimate is consistent with $K$.

In order to get one unweighted approximate sample, we can resample according to the weights: $i \sim \mathrm{Cat}(w_{1:K}); (x, c) \leftarrow (x^i, c^i)$.
Alternatively, we can compute the posterior mean by setting the test function $f$ above to be the identity function.
In practice, due to severe weight degeneracy, we find both alternatives to produce nearly identical outputs.

\begin{algorithm}[!tb]
   \caption{ReGAL with Importance Sampling}
   \label{alg:is_regal}
\begin{algorithmic}[1]
   \STATE {\bfseries Input:} Diffusion prior denoiser $\epsilon_\phi(z_t; t)$, prior over the auxiliary latent $p(c | x)$, likelihood $p(y | x, c)$, Langevin step size $\delta$, number of discretization bins $T$ and a discretization schedule $s(i) = (i - 1) / T, t(i) = i / T$, number of importance samples $K$, observation $y$.
   \STATE {\bfseries Output:} A set of weighted particles $\{((x^k, c^k), w_k)\}_{k = 1}^K$ that approximates $p(x, c | y) \propto p(x)p(c | x)p(y | x, c)$.
   \FOR{$k = 1, \dotsc, K$}
   \STATE Sample $z_1^k \sim p(z_1) = \mathcal N(z_1; 0, 1)$ and $c_1^k \sim p(c | z_1)$.
   \FOR{$i = T, \dotsc, 1$}
   \STATE Sample $z_{s(i)}^k \sim \hat p(z_{s(i)} | z_{t(i)}^k, y)$ (see \eqref{eq:reconstruction_guidance_conditional}).
   \STATE Sample $c_{s(i)}^k \sim \hat p(c_{s(i)} | z_{s(i)}^k, c_{t(i)}, y; \delta)$ (see \eqref{eq:langevin_step}).
   \ENDFOR
   \STATE Sample $x^k \sim p(x | z_0^k), c^k \sim \hat p(c | x^k, c_0^k, y; \delta)$ (see \eqref{eq:langevin_step}).
   \STATE Evaluate weight $w_k$ according to \eqref{eq:is_regal_weight}.
   \ENDFOR
   \STATE {\bfseries Return:} $\{((x^k, c^k), w_k)\}_{k = 1}^K$.
\end{algorithmic}
\end{algorithm}

\section{ReGAL with Sequential Monte Carlo}
\label{app:smc_regal}

Given the sequential nature of ReGAL, it is natural to introduce place it within a sequential Monte Carlo (SMC) framework \citep{doucet2009tutorial,chopin2020introduction}. The resultant algorithm, dubbed ReGAL-SMC, generalizes the previous versions.

ReGAL samples variables in the following order
\begin{align*}
    (z_1, c_1), (z_{t(T - 1)}, c_{t(T - 1)}), \dotsc, (z_{t(i)}, c_{t(i)}), (z_{s(i)}, c_{s(i)}), \dotsc, (z_{t(1)}, c_{t(1)}), (z_0, c_0), (x, c) %
\end{align*}
where the full sequence has length $T + 2$.
To simplify notation for SMC, let's re-define this sequence as $((z^{(n)}, c^{(n)}))_{n = 1}^N$ where $N = T + 2$.
That is, $(z^{(n)}, c^{(n)}) := (z_{t(T - n + 1)}, c_{t(T - n + 1)})$ for $n = 1, \dotsc, N - 1$ and $(z^{(N)}, c^{(N)}) := (x, c)$.

\textbf{Final target distribution}
Using this notation, we can rewrite the target distribution on the full extended space $z^{(1:N)}, c^{(1:N)}$ from \eqref{eq:regal_is_target} as
\begin{align}
    \gamma_N(z^{(1:N)}, c^{(1:N)}) := p_{\text{diff}}(z^{(1)})p_{\text{aux}}(c^{(1)} | z^{(1)}) \left(\prod_{n = 2}^N p_{\text{diff}}(z^{(n)} | z^{(n - 1)}) p_{\text{aux}}(c^{(n)} | z^{(n)})\right) p(y | z^{(N)}, c^{(N)}) \label{eq:regal_smc_final_target}
\end{align}
where the density terms correspond to the terms in \eqref{eq:regal_is_target}.

Given a sequence of variables $z^{(1:N)}, c^{(1:N)}$, SMC sequentially builds approximations of target distributions $\gamma_n(z^{(1:n)}, c^{(1:n)})$ from $n = 1$ to $N$.
Since the final target distribution $\gamma_N$ in \eqref{eq:regal_smc_final_target} marginalizes to the desired posterior distribution $p(x, c | y)$, given an approximate sample $(z^{(1:N)}, c^{(1:N)})$ from $\gamma_N$, we can take $(x, c) := (z^{(N)}, c^{(N)})$ to be the approximate sample from $p(x, c | y)$.

\textbf{Intermediate target distributions}
While there is complete freedom in defining the target densities $\gamma_n(z^{(1:n)}, c^{(1:n)})$ when $n < N$, some choices lead to better posterior approximations.
A straightforward choice would be defining $\gamma_n$ exactly the same as \eqref{eq:regal_smc_final_target}, with the product going only until $n$ and without the likelihood term.
This choice is problematic since the likelihood term is incorporated only at the last step so the target distributions are equal to the prior for all but the last step.
Thus, SMC makes no progress in approximating the posterior for the first $N - 1$ steps, and is asked to do the full inference in one step.

A better sequence of targets would include the prior times the \emph{lookahead likelihood}, $p(y | z^{(n)}, c^{(n)}) := \int p(y | z^{(N)}, c^{(N)}) \prod_{i = n + 1}^N p(z^{(i)}, c^{(i)} | z^{(i - 1)}, c^{(i - 1)}) \,\mathrm dz^{(n + 1:N)} dc^{(n + 1:N)}$.
In this case, SMC targets the \emph{smoothing} posterior $\gamma_n^{\text{opt}}(z^{(1:n)}, c^{(1:n)}) := \int \gamma_N(z_{1:N}, c_{1:N}) \,\mathrm dz^{(n + 1:N)}$ which yields a much more gradual transition when $n$ goes from 1 to $N$.
Since the lookahead likelihood is unavailable, we approximate it as
\begin{align}
    \tilde p(y | z^{(n)}, c^{(n)}) := p(y | x = x_\phi(z_t = z^{(n)}; t = t(T - n + 1)), c = c^{(n)}) \approx p(y | z^{(n)}, c^{(n)})
\end{align}
like in \eqref{eq:regal_score}.
Thus, the actual intermediate target density (that approximates $\gamma_n^{\text{opt}}$) we use in ReGAL-SMC is
\begin{align}
    \gamma_n(z^{(1:n)}, c^{(1:n)}) := p_{\text{diff}}(z^{(1)})p_{\text{aux}}(c^{(1)} | z^{(1)}) \left(\prod_{i = 2}^n p_{\text{diff}}(z^{(i)} | z^{(i - 1)}) p_{\text{aux}}(c^{(i)} | z^{(i)})\right) \tilde p(y | z^{(n)}, c^{(n)}), \quad (n = 1, \dotsc, N - 1). \label{eq:regal_smc_intermediate_target}
\end{align}

\textbf{Proposals and weights} We can factorize the proposal distribution in \eqref{eq:regal_is_proposal} as a Markov chain of per-step proposals $q(z^{(1:N)}, c^{(1:N)}) = \prod_{n = 1}^N q_n(z^{(n)}, c^{(n)} | z^{(n - 1)}, c^{(n - 1)})$ where
\begin{align}
    q_n(z^{(n)}, c^{(n)} | z^{(n - 1)}, c^{(n - 1)}) =
    \begin{cases}
    p_{\text{diff}}(z^{(1)}) p_{\text{aux}}(c^{(1)} | z^{(1)}) & \text{ if } n = 1 \\
    \hat p(z^{(n)} | z^{(n - 1)}, y) \hat p(c^{(n)} | z^{(n)}, c^{(n - 1)}, y) & \text{ if } 1 < n < N\\
    p_{\text{diff}}(z^{(N)} | z^{(N - 1)}) \hat p(c^{(N)} | z^{(N)}, c^{(N - 1)}, y) & \text{ if } n = N.
    \end{cases} \label{eq:regal_smc_proposal}
\end{align}
\Cref{alg:smc_regal} shows the full algorithm, where the weights are computed as
\begin{align}
    w^{(1)}(z^{(1)}, c^{(1)}) &= \frac{\gamma_1(z^{(1)}, c^{(1)})}{q_n(z^{(1)}, c^{(1)})} \\
    &= \frac{p_{\text{diff}}(z^{(1)})p_{\text{aux}}(c^{(1)} | z^{(1)})\tilde p(y | z^{(1)}, c^{(1)})}{p_{\text{diff}}(z^{(1)}) p_{\text{aux}}(c^{(1)} | z^{(1)})} \\
    &= \tilde p(y | z^{(1)}, c^{(1)}), \label{eq:first_weight}\\
    w^{(n)}(z^{(1:n)}, c^{(1:n)}) &= \frac{\gamma_n(z^{(1:n)}, c^{(1:n)})}{\gamma_{n - 1}(z^{(1:n - 1)}, c^{(1:n - 1)}) q_n(z^{(n)}, c^{(n)} | z^{(n - 1)}, c^{(n - 1)})} \\
    &=
    \begin{cases}
        \frac{\tilde p(y | z^{(n)}, c^{(n)})p_{\text{diff}}(z^{(n)} | z^{(n - 1)}) p_{\text{aux}}(c^{(n)} | z^{(n)})}{\tilde p(y | z^{(n - 1)}, c^{(n - 1)})\hat p(z^{(n)} | z^{(n - 1)}, y) \hat p(c^{(n)} | z^{(n)}, c^{(n - 1)}, y)} & \text{if } 1 < n < N\\
        \frac{\tilde p(y | z^{(N)}, c^{(N)}) p_{\text{aux}}(c^{(N)} | z^{(N)})}{\tilde p(y | z^{(N - 1)}, c^{(N - 1)}) \hat p(c^{(N)} | z^{(N)}, c^{(N - 1)}, y)} & \text{if } n = N.
    \end{cases} \label{eq:non_first_weight}
\end{align}

\textbf{At optimality} Proposals in \eqref{eq:regal_smc_proposal} are meant to approximate the conditional posterior $q_n^{\text{opt}}(z^{(n)}, c^{(n)} | z^{(n - 1)}, c^{(n - 1)}, y) \propto \gamma_n^{\text{opt}}(z^{(1:n)}, c^{(1:n)}) / \gamma_{n - 1}^{\text{opt}}(z^{(1:n - 1)}, c^{(1:n - 1)})$.
If $q_n = q_n^{\text{opt}}$, sequentially sampling from $q_n$ yields an exact posterior sample.
Moreover, if $\gamma_n = \gamma_n^{\text{opt}}$, SMC's weights are constant at every step, thus avoiding any resampling.
In practice, our approximations to $q_n^{\text{opt}}$ and $\gamma_n^{\text{opt}}$ are not exact, requiring the weights and resampling in order to make SMC's posterior approximation consistent.

\begin{algorithm}[!tb]
   \caption{ReGAL with Sequential Monte Carlo}
   \label{alg:smc_regal}
\begin{algorithmic}[1]
   \STATE {\bfseries Input:} Diffusion prior denoiser $\epsilon_\phi(z_t; t)$, prior over the auxiliary latent $p(c | x)$, likelihood $p(y | x, c)$, Langevin step size $\delta$, number of discretization bins $T$ and a discretization schedule $s(i) = (i - 1) / T, t(i) = i / T$, number of importance samples $K$, observation $y$.
   \STATE {\bfseries Output:} A set of weighted particles $\{((x^k, c^k), w_k)\}_{k = 1}^K$ that approximates $p(x, c | y) \propto p(x)p(c | x)p(y | x, c)$.

   \STATE Propose $z^{(1), k}, c^{(1), k} \sim q_1(z^{(1)}, c^{(1)})$ ($k = 1, \dotsc, K$) (see \eqref{eq:regal_smc_proposal}).
   \STATE Evaluate weight $w^{(1), k} = w^{(1)}(z^{(1), k}, c^{(1), k})$ ($k = 1, \dotsc, K$) (see \eqref{eq:first_weight}).
   
   \FOR{$n = 2, \dotsc, N$}
   \STATE Sample ancestral indices $a^{(n - 1), 1:K}$ based on weights $w^{(n - 1), 1:K}$.
   \STATE Set $z^{(n - 1), k}, c^{(n - 1), k} \leftarrow z^{(n - 1), a^{(n - 1), k}}, c^{(n - 1), a^{(n - 1), k}}$ ($k = 1, \dotsc, K$).
   
   \STATE Propose $z^{(n), k}, c^{(n), k} \sim q_n(z^{(n)}, c^{(n)} | z^{(n - 1), k}, c^{(n - 1), k})$ ($k = 1, \dotsc, K$) (see \eqref{eq:regal_smc_proposal}).
   \STATE Evaluate weight $w^{(n), k} = w^{(n)}(z^{(1:n), k}, c^{(1:n), k})$ ($k = 1, \dotsc, K$) (see \eqref{eq:non_first_weight}).
   \ENDFOR
   \STATE {\bfseries Return:} $\{((x^{(N), k}, c^{(N), k}), w^{(N), k})\}_{k = 1}^K$.
\end{algorithmic}
\end{algorithm}

\section{Scene Prior Architecture Details}
\label{app:priors}

When computing the likelihood, we subsample 1024 all the rays across all views of a scene, with the number primarily motivated by its effect on inference speed.

\textbf{ProbNeRF} We use the same architecture as \citet{hoffman2023probnerf}.

\textbf{Triplane SSDNeRF} The triplane NeRF we use has 3 planes with shape $[128, 128, 6]$.
Samples from those planes are trilinearly interpolated, concatenated and passed through one hidden layer of size 64 for the density output, and another layer of the same size for the color output.
We use the ReLU activation the hidden layers.
For the color output we additionally concatenate positional encodings for the ray direction.
We use the sigmoid activation for color, and softplus for density.
We additionally bias the pre-activation density input by $-3$, so improve the gradient flow to all parts of the NeRF.
We deliberately use relatively lightweight parameterization for the NeRF, so as to force the prior to do model the distribution as adjacent to the observations as possible.

For the denoiser, we use a UNet with the following levels: $[128, 128, 128] \rightarrow [64, 64, 256] \rightarrow [32, 32, 256] \rightarrow [16, 16, 512] \rightarrow [8, 8, 512]$.
The final 3 levels also get a spatial attention module with 4 heads.

\textbf{Set-Latent SSDNeRF} For the set-latent NeRF, we use a similar architecture as \citet{sajjadi2022scene}. We use 512 latents, each with 64 dimensions. The transformer uses 8 heads, with the QKV dimension set to 256 and MLP dimension to 256.
After the transformer, we use the same stack of layers and conditioning signals as with the TriPlane NeRF described above. For MSN we set the radius of the NeRF to 15, outside of which we set the density to 0.

For the denoiser, we use a transformer with 12 layers and 16 heads, and QKV dimension set to 1024 and MLP dimension to 2048. The inputs are first projected to a dimension of 512.

\section{Scene Prior Training Details}
\label{app:training}

\textbf{ProbNeRF} We train for $2\times10^6$ steps.
We use the Adam \citep{kingma2017adam} optimizer with a learning rate schedule where we warm up the learning rate from 0 to $10^{-4}$ over 50 steps, and then step-wise halve it every 50000 steps afterward.
We used a minibatch of 8 scenes. For the guide, we use 10 random views per scene.

\textbf{SSDNeRF} We train for $5 \times 10^5$ steps.
For ShapeNet, we use the Adam optimizer with a learning rate schedule where we warm up the learning rate from 0 to $10^{-3}$ and then step-wise halve it every 125000 steps afterward.
For MSN, we use the cosine learning rate schedule \citep{loshchilov2017sgdr}, where we warm up the learning rate from 0 to $10^{-3}$ and then decay it to 0 over the training duration.
We use a minibatch of 16 scenes.
For MSN, we found it necessary to clip the global gradient norm to 1.

For the scene latents, they're initialized at zeros for ShapeNet and from an isotropic Gaussian with a scale of 0.1.
It is important to initialize the latents away from zero when using a transformer-based denoiser.
We use L2 regularization for the latents, with a factor of $10$ for ShapeNet and $30$ for MSN.
We anneal the number of latent optimization substeps from 15 to 4 at iteration 10000, and, for MSN, to 1 at iteration 100000.
Generally, using fewer substeps lets the diffusion prior have an easier job fitting the distribution of the learned scene latents.

\section{Scene Prior Evaluation Details}
\label{app:evaluation}

\textbf{ShapeNet Masks} Kubric provides ground truth segmentation masks for rendered objects.
For ShapeNet we found it necessary to also clip the regions which had depths greater than 1000.

\textbf{Corruption Parameterization} For the NeRF corruption, we use the same NeRF as ProbNeRF.
For the FOV parameterization we perform inference over $\mathbb{R}$, but constrain it to lie in $[\pi/4, 3\pi/4]$ via the sigmoid transformation.

\textbf{ProbNeRF MAP} We use the Adam optimizer with learning rate of $10^{-2}$ for $3000$ iterations.

\textbf{ProbNeRF VI} We use the Adam optimizer with learning rate of $10^{-4}$ for $10000$ iterations.
A lower learning rate is necessary for stability.
We anneal the KL divergence term \citep{kingma2019introduction} with a schedule that linearly increases from 0 to 1 over a period of $5000$ iterations.

\textbf{Diffusion Conditioning} We use a schedule to anneal the likelihood $f(t) = 0.01 + 0.99 (1 - t)^{1.5}$.
Additionally, for Rain and Cloud corruptions for MSN, we scale the likelihood when computing $\hat p(z_s | z_t, y)$ \Cref{eq:reconstruction_guidance_conditional} by a factor of 3.
This is to encourage the scene parameters to explain the scene rather than the corruption NeRF.
We expect this to be less necessary with a stronger prior, and indeed we did not need to do this for ShapeNet.

For ShapeNet, we use $T = 500$ for clean and rain conditions, $T = 5000$ for clouds and $T = 1500$ for FOV.
For MSN, we use $T = 2000$ steps in all cases except FOV, where we use $T = 5000$. In general, harder corruptions require more steps, but due to computational requirements we preferred to use as few steps as feasible.

For ShapeNet, we use $\delta=10^{-5}$ for rain and cloud conditions, $\delta=10^{-6}$ for FOV.
For MSN, we use $\delta=3\times10^{-5}$ for rain and cloud conditions, $\delta=10^{-5}$ for FOV.
Up to a limit, using a larger step size encourages mixing of the SMC rejuvenation kernel.

\section{Prior Samples}
\label{app:prior_samples}

In \Cref{fig:probnerf_shapenet_prior_samples,fig:shapenet_prior_samples,fig:msn_prior_samples}, we show unconditional samples from the scene priors used in the experiments.
The ShapeNet car samples contain much finer details under the SSDNeRF model.
While MultiShapeNet samples are not as high quality, the underlying scene prior still proves to empirically outperform baselines in the robust inverse graphics tasks in the experiments.

\begin{figure*}[!ht]
  \centering
    \includegraphics[width=\textwidth]{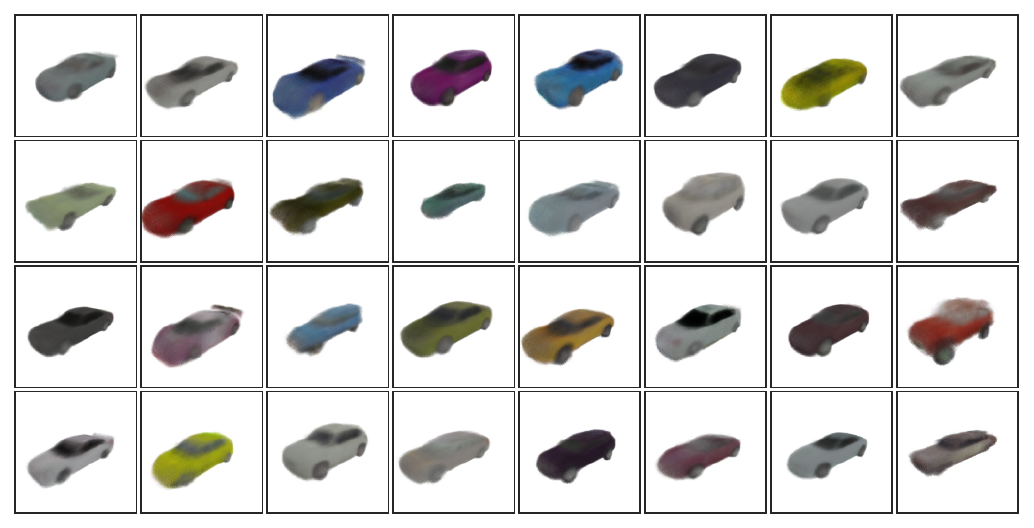}
\caption{
    ShapeNet cars prior samples from the ProbNeRF model.
}
    \label{fig:probnerf_shapenet_prior_samples}
\end{figure*}
\begin{figure*}[!ht]
  \centering
    \includegraphics[width=\textwidth]{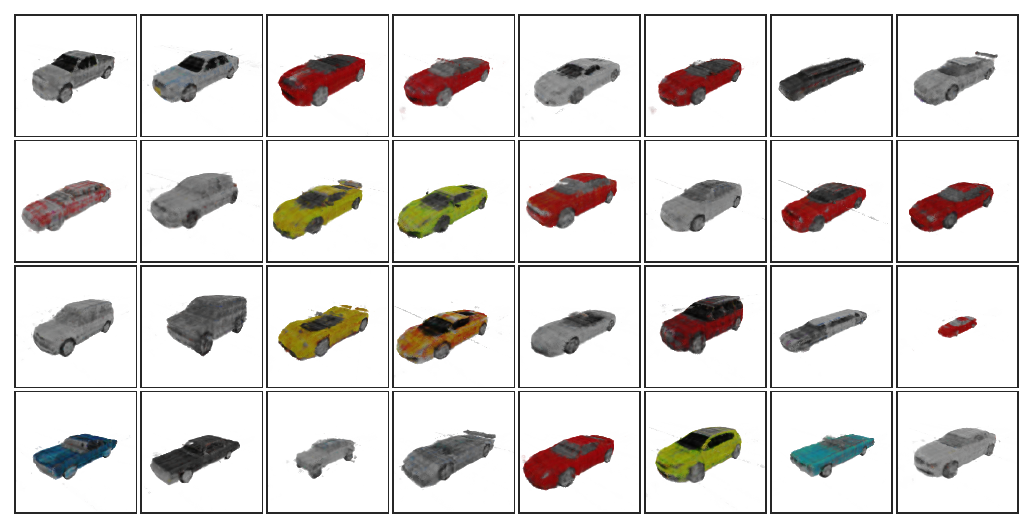}
\caption{
    ShapeNet cars prior samples from the SSDNeRF model.
}
    \label{fig:shapenet_prior_samples}
\end{figure*}
\begin{figure*}[!ht]
  \centering
    \includegraphics[width=\textwidth]{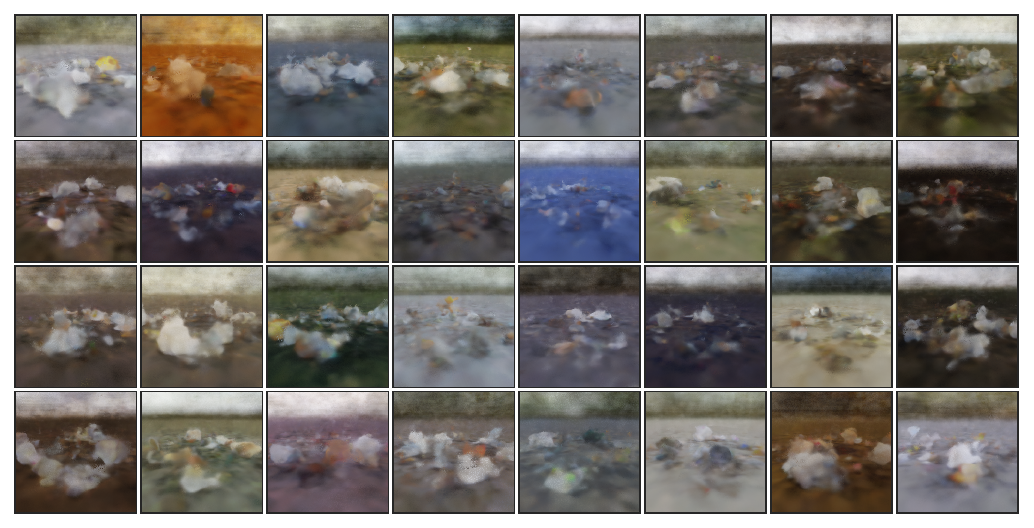}
\caption{
    MultiShapeNet prior samples from the SSDNeRF model.
}
    \label{fig:msn_prior_samples}
\end{figure*}

\section{Additional Reconstructions}
\label{app:extra_recons}

See \Cref{fig:shapenet_recon_rain,fig:shapenet_recon_clouds,fig:msn_recon_rain,fig:msn_recon_clouds}.

\begin{figure*}[!ht]
  \centering
    \includegraphics[width=0.8\textwidth]{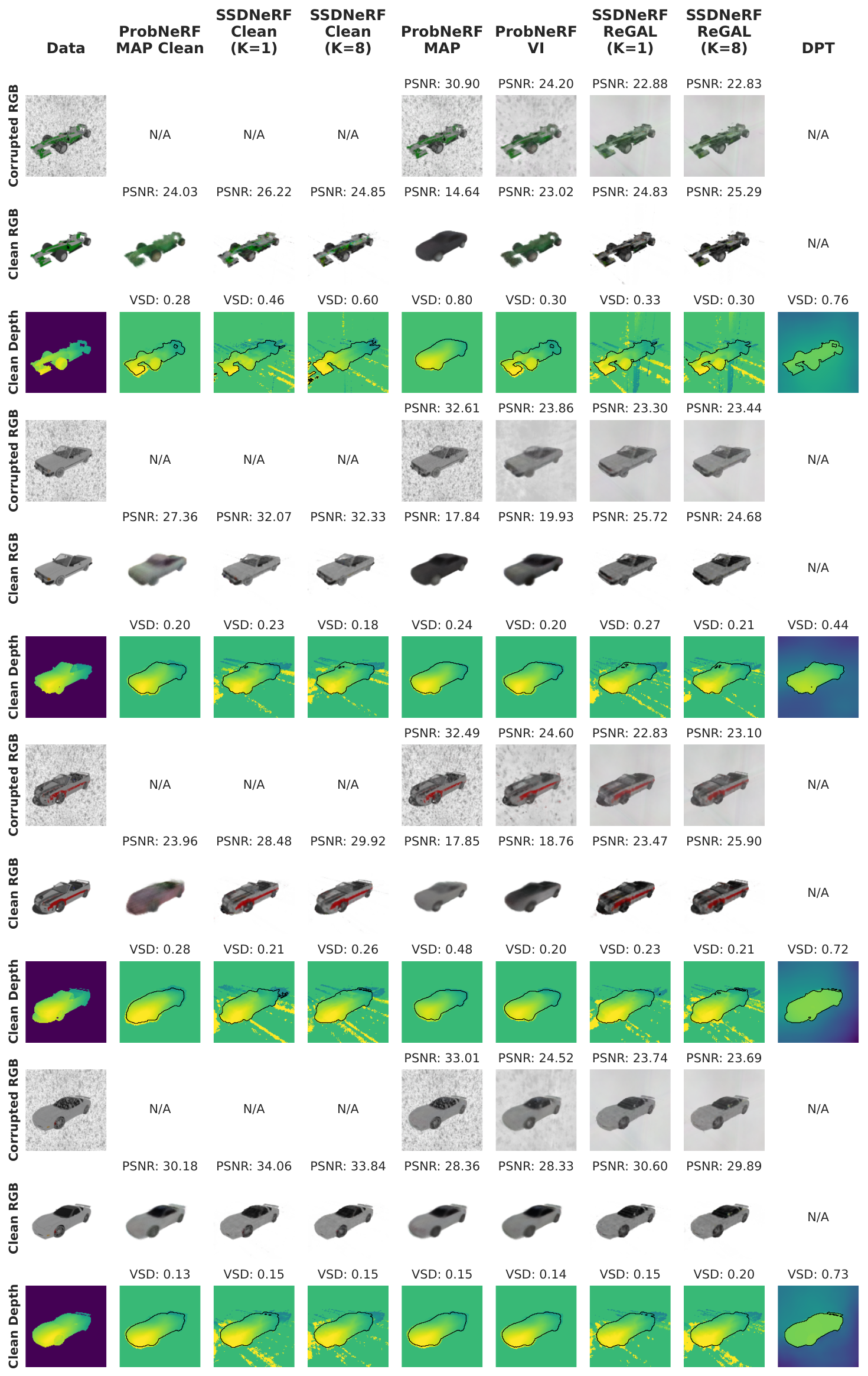}
\caption{
    Reconstructions for the rain corruption of the ShapeNet cars dataset. See \cref{fig:recons} for details.
}
\label{fig:shapenet_recon_rain}
\end{figure*}

\begin{figure*}[!ht]
  \centering
    \includegraphics[width=0.8\textwidth]{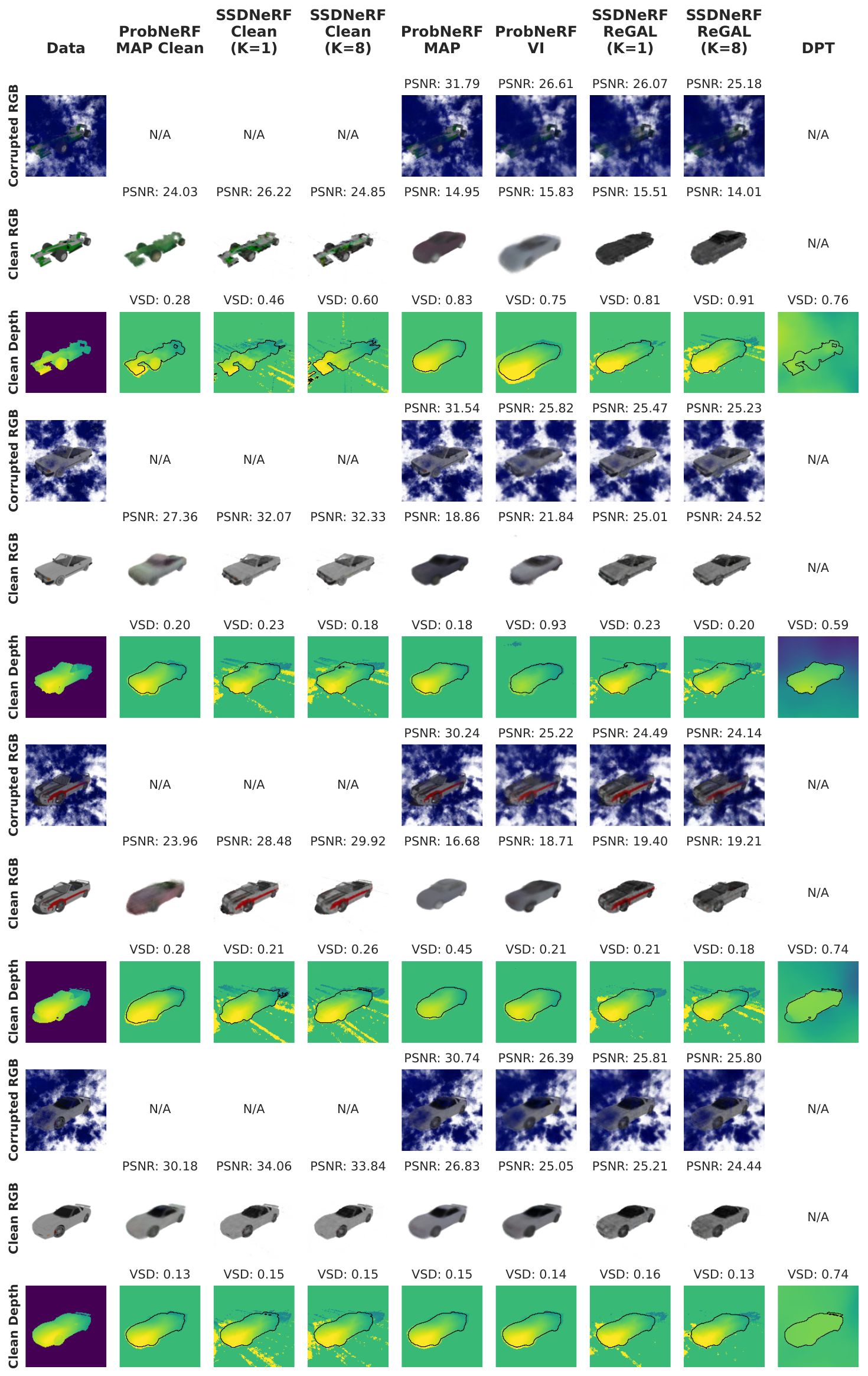}
\caption{
    Reconstructions for the clouds corruption of the ShapeNet cars dataset. See \cref{fig:recons} for details.
}
\label{fig:shapenet_recon_clouds}
\end{figure*}

\begin{figure*}[!ht]
  \centering
    \includegraphics[width=0.52\textwidth]{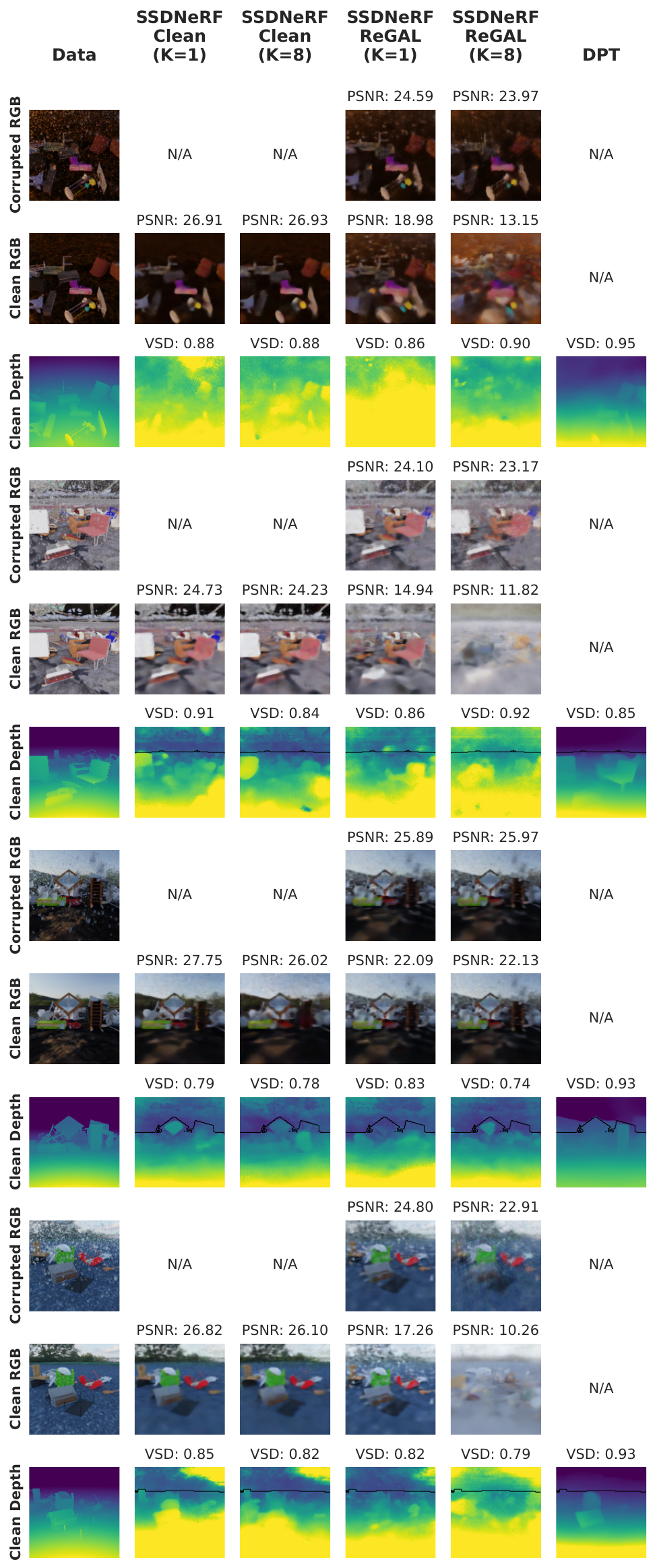}
\caption{
    Reconstructions for the rain corruption of the MultiShapeNet dataset. See \cref{fig:recons} for details.
}
\label{fig:msn_recon_rain}
\end{figure*}

\begin{figure*}[!ht]
  \centering
    \includegraphics[width=0.52\textwidth]{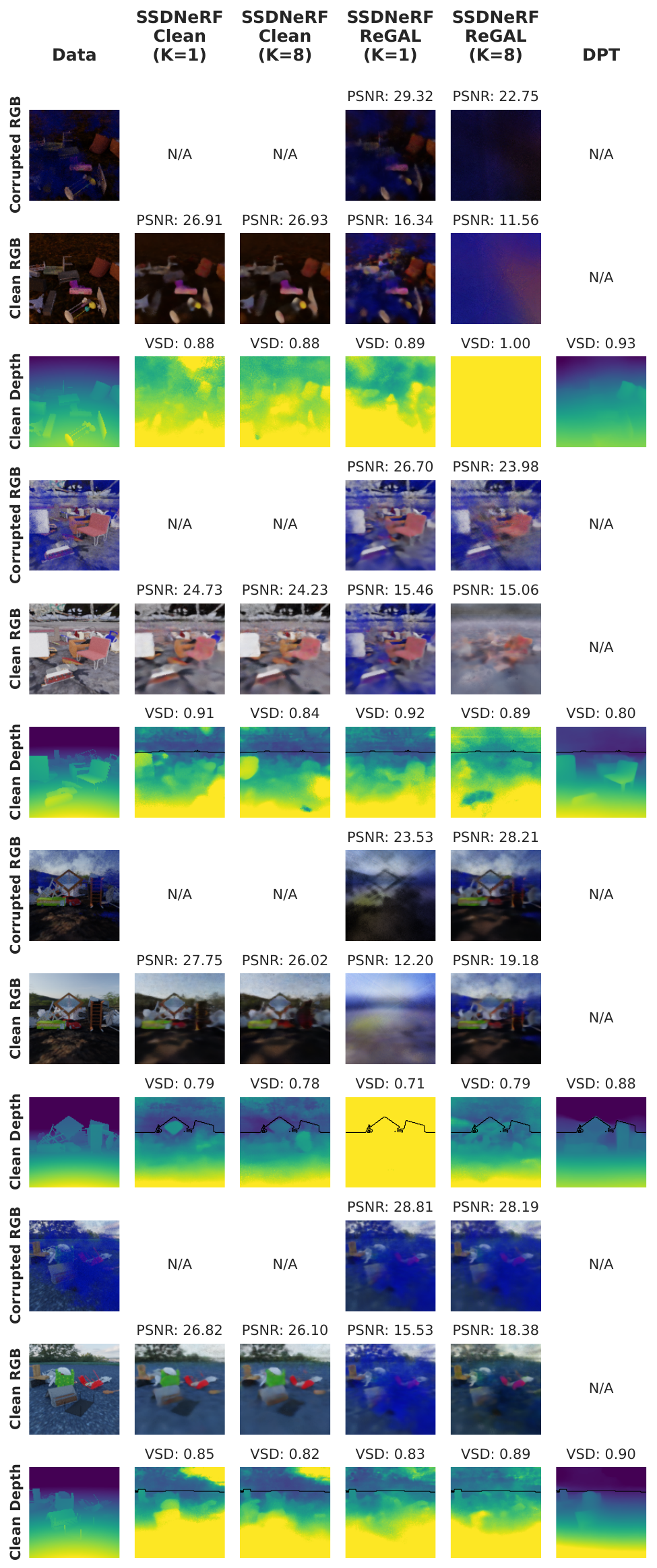}
\caption{
    Reconstructions for the clouds corruption of the MultiShapeNet. See \cref{fig:recons} for details.
}
\label{fig:msn_recon_clouds}
\end{figure*}

\section{Posterior Uncertainty}
\label{app:posterior_uncertainty}

In \Cref{fig:shapenet_recon_clouds_samples} we show multiple samples from the posterior for the clouds corruption of the ShapeNet cars dataset for SSDNeRF ReGAL ($K=8$).

\begin{figure*}[!ht]
  \centering
    \includegraphics[width=0.52\textwidth]{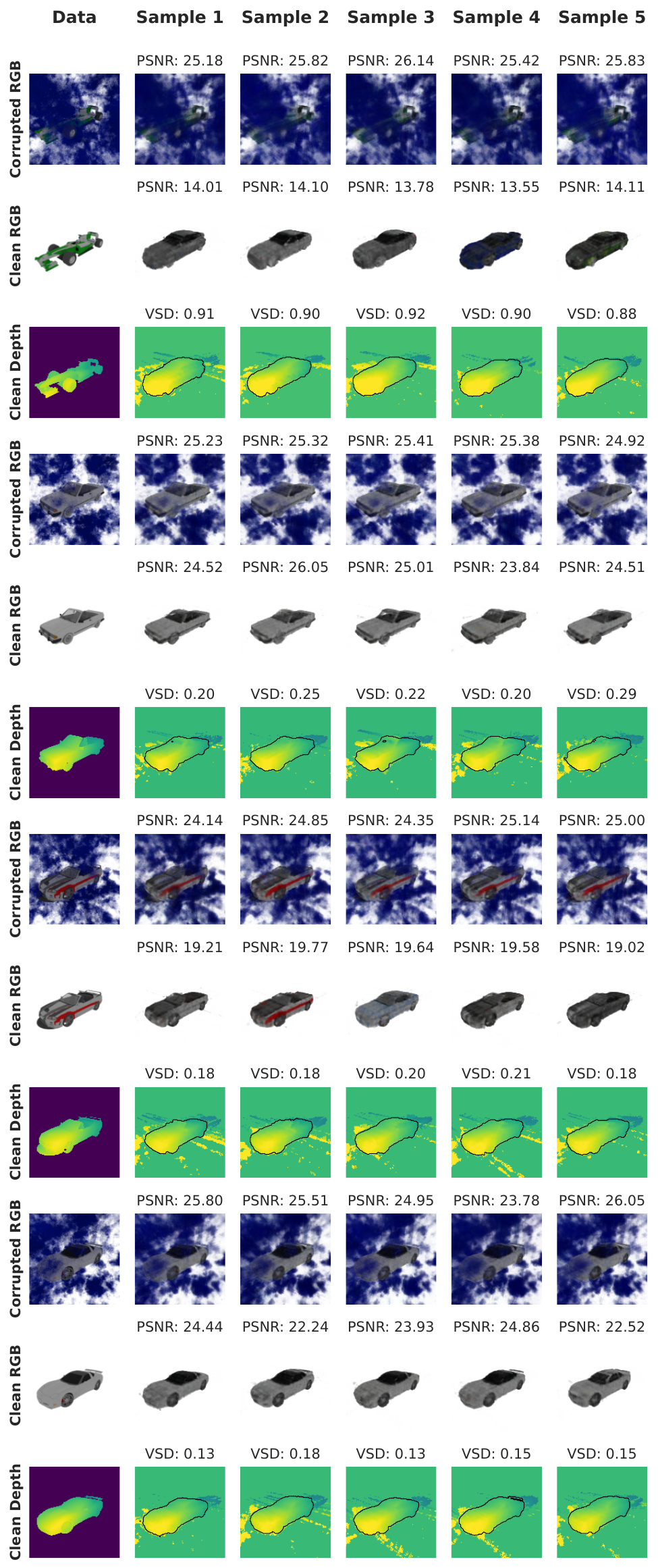}
\caption{
    Multiple samples from the posterior for the clouds corruption of the ShapeNet cars dataset for SSDNeRF ReGAL (K=8). See \cref{fig:recons} for details.
}
\label{fig:shapenet_recon_clouds_samples}
\end{figure*}

\section{Performance}
\label{app:performance}

On a single A100 GPU, for each image it takes 9.5 minutes for MultiShapeNet and 7.5 minutes for ShapeNet to run ReGAL for 2000 steps to generate 8 particles. The time is proportional to the number of steps, so at a cost of quality, it could be reduced (as much as 10x before complete breakdown). The particles can be simulated on separate GPUs as well with relatively small cross-GPU communication requirements (due to IS/SMC resampling). Ultimately, however, practical deployment of this method would require more advanced diffusion sampling methods (such as Consistency Models \citep{song2023consistency} and Flow Matching \citep{lipman2022flow}) which would drastically reduce the number of steps we need to take. It is likely these are not drop-in and would require adjustment of the ReGAL algorithm.

\end{document}

%% file: math_commands.tex
\usepackage{amsmath,amsfonts,bm}

\def\1{\bm{1}}

\DeclareMathAlphabet{\mathsfit}{\encodingdefault}{\sfdefault}{m}{sl}
\SetMathAlphabet{\mathsfit}{bold}{\encodingdefault}{\sfdefault}{bx}{n}

\newcommand{\E}{\mathbb{E}}

\DeclareMathOperator*{\argmax}{arg\,max}

\newcommand{\ELBO}{\mathrm{ELBO}}